\documentclass{article}

\usepackage[final]{colm2026_conference}
\usepackage{url}

\usepackage{microtype}
\usepackage{graphicx}
\usepackage{booktabs}
\usepackage{multirow}
\usepackage{pifont}

\usepackage{amsmath,amssymb,amsthm}
\usepackage{hyperref}
\usepackage[capitalize,noabbrev]{cleveref}
\usepackage{bm}
\usepackage{algorithm}
\usepackage{algorithmic}
\usepackage{subcaption}
\usepackage{float}
\usepackage{xcolor}
\usepackage{tcolorbox}
\tcbuselibrary{skins,breakable}
\usepackage{pgfplots}
\pgfplotsset{compat=1.17}
\usepackage{tikz}
\usepackage{lineno}

\definecolor{color1}{HTML}{2E86AB}
\definecolor{color2}{HTML}{A23B72}
\definecolor{color3}{HTML}{F18F01}
\definecolor{color4}{HTML}{C73E1D}
\definecolor{color5}{HTML}{3B1F2B}
\definecolor{color6}{HTML}{95C623}

\definecolor{tokred}{HTML}{B31B1B}
\definecolor{tokredlight}{HTML}{FDDEDE}
\definecolor{altblue}{HTML}{1F60C4}
\definecolor{promptgray}{HTML}{F3F3F3}
\definecolor{responsegray}{HTML}{FAFAFA}

\newcommand{\tokdiff}[2]{%
  \colorbox{tokredlight}{\color{tokred}\texttt{#1}}%
  {\color{altblue}\textsuperscript{\texttt{#2}}}%
}

\newtcolorbox{promptframe}{
  colback=promptgray, colframe=black!40, boxrule=0.4pt,
  fontupper=\small\sffamily, left=4pt, right=4pt, top=2pt, bottom=2pt,
  title={\footnotesize\bfseries Prompt}, coltitle=black,
  attach boxed title to top left={xshift=4pt,yshift=-2mm},
  boxed title style={colback=promptgray, colframe=black!40, boxrule=0.4pt}
}

\newtcolorbox{responseframe}{
  colback=responsegray, colframe=black!20, boxrule=0.3pt,
  fontupper=\scriptsize\ttfamily, left=4pt, right=4pt, top=3pt, bottom=3pt,
  breakable
}

\theoremstyle{plain}
\newtheorem{theorem}{Theorem}[section]
\newtheorem{lemma}[theorem]{Lemma}

\theoremstyle{definition}
\newtheorem{definition}[theorem]{Definition}

\theoremstyle{remark}

\title{Statistically-Lossless Quantization of Large Language Models}

\author{Michael Helcig \\
ETH Zürich
\And
Eldar Kurtić \\
Red Hat AI \& ISTA
\And
Dan Alistarh \\
ISTA
}
\begin{document}

\ifcolmsubmission
\linenumbers
\fi

\maketitle

\lhead{Published as a conference paper at COLM 2026}

\begin{abstract}
Model quantization has become essential for efficient large language model deployment, yet existing approaches present clear trade-offs: methods such as GPTQ and AWQ achieve practical compression but are \emph{lossy}, while \textit{lossless} techniques preserve fidelity but lack inference acceleration. This paper explores the middle ground of \textit{statistically-lossless compression}, examining three complementary aspects of what losslessness means for quantized LLMs. First, \emph{task-lossless} compression preserves zero-shot benchmark accuracy within natural sampling variance and is achievable at aggressive bitwidths. Second, we formalize the stricter notion of \emph{distribution-lossless} compression, requiring the quantized model's next-token distribution to be practically indistinguishable from the original, and propose the \textit{Expected Acceptance Rate (EAR)}, the maximum token-agreement probability under optimal coupling, as a directly interpretable fidelity metric (e.g.\ EAR $\geq 0.99$ means 99\% agreement). Third, we prove a \textit{gamma-squared variance law} showing that symmetric quantization inflates noise variance by $\gamma^2$ relative to asymmetric, making asymmetric quantization a prerequisite for distribution-lossless fidelity but not for task-level preservation. Through \textit{SLQ}, a layer-wise non-uniform method with asymmetric quantization and wide bitwidth search, we obtain task-lossless compression at well below 4 bits per parameter (as low as 3.3 bits depending on the model), distribution-lossless compression at 5--6 bits per parameter on average, and inference speedups of $1.7$--$3.7\times$ compared to FP16 using optimized kernels. Source code is available at \url{https://github.com/IST-DASLab/SLQ}.
\end{abstract}


\section{Introduction}
\label{sec:introduction}

The large computational and memory costs of large language model (LLM) inference have led \textit{model quantization}, i.e. the reduction in bit-width for weights, activations, or both, to become a standard technique for efficient deployment. 
As such, over the past few years, hundreds of LLM quantization techniques have been proposed, e.g.~\citep{dettmers2022llmint8, frantar2023gptq, chee2023quip, tseng2024qtip, kim2024squeezellm, malinovskii2025higgs, lin2024awq}, exploring the trade-off curve between compression achieved and accuracy loss. 

From the accuracy/compression perspective, existing approaches can be clustered into two categories. 
The first is represented by \emph{lossy compression techniques}, such as Round-to-Nearest (RTN) quantization~\citep{dettmers2022llmint8}, llama.cpp~\citep{llama.cpp},  GPTQ~\citep{frantar2023gptq}, or AWQ~\citep{lin2024awq} which seek to map existing models to popular hardware-supported formats, such as 4-bit grouped weight quantization. These techniques have become extremely popular, and have widespread practical deployment, although they do decrease model accuracy: for example, a recent study by~\citet{Kurtic2025BF16orDeath} observed that methods such as GPTQ and AWQ can drop zero-shot accuracy on some complex tasks by 3-5\% in the case of 4-bit weight quantization. More complex vector-quantized~\citep{egiazarian2024aqlm, tseng2024quipsharp} or trellis-based~\citep{tseng2024qtip} methods can offer reduced accuracy loss at even lower bit-widths, but are hard to apply and interface with state-of-the-art inference pipelines such as vLLM~\citep{kwon2023vllm}, where support for such complex techniques is currently lacking. 

The second, more recent, category is given by methods which aim explicitly for \textit{lossless quantization} of the original weights. For instance, ZipNN~\citep{hershcovitch2025zipnn} observes that the exponents of LLM weights are highly compressible, as they do not use the entire bit-width range, reducing average size to approximately 11 bits per parameter (bpp) from 16 bpp, without accuracy loss. Other approaches include entropy-coded FP8 representations~\citep{yang2025compress} and dynamic-length BF16 compression~\citep{zhang2025dfloat11}. On the downside, lossless techniques struggle to reach high compression ratios and require complex kernel support for fast decoding at inference time. 

\noindent\textbf{Our Approach.} In this paper, we focus on the gap between these two types of approaches, and investigate \textit{Statistically Lossless LLM Quantization (SLQ)}, which we define as compression that may deviate from perfectly matching outputs under deterministic decoding, but maintains the model's output characteristics under standard sampling parameters.
Our approach is motivated by the fact that most open-weight models, such as Qwen~\citep{bai2023qwen} or Llama~\citep{dubey2024llama3}, come with recommended \textit{stochastic} sampling parameters (e.g.\ temperature $= 0.7$ and $\textnormal{top}_k = 20$ for Qwen3), under which even \textit{unquantized models' outputs vary} by a relative ${\sim}1\%$ on zero-shot or user evaluations.
The question we ask is: \textbf{Is it possible to near-losslessly compress LLMs such that their outputs \textit{statistically} match the original?} 

We examine this question along two complementary axes: 
\begin{itemize}
    \item ~\emph{Task-lossless} (TL) quantization, where zero-shot benchmark accuracy is preserved within sampling variance. 
    \item ~\emph{Distribution-lossless} (DL) quantization, a stricter criterion where the quantized model's \textit{output distribution} is practically indistinguishable from the original. 
\end{itemize}

Overall, we show that statistically lossless quantization matches realistic deployment scenarios while allowing significant gains beyond lossless compression and inference acceleration.
In addition, one key technical finding in the context of SLQ is the role of \textit{asymmetric quantization}, i.e. the presence of a non-trivial zero-point parameter for the quantization grid, as a necessary design choice for achieving Distribution-lossless (DL) fidelity. 

\begin{itemize}
    \item We formalize the distinction between TL and DL compression for LLMs and propose the \textit{expected acceptance rate (EAR)} metric, measuring the maximum token-agreement probability between original and quantized distributions under optimal coupling, as an interpretable output fidelity metric complementary to standard KL-divergence.

    \item We prove an error gap between symmetric and asymmetric quantization grids. For a weight distribution on $[L, U]$ we define the \emph{centering inefficiency} $\gamma = 2M/R$, with dynamic range $R = U - L$ and $M = \max(|L|, |U|)$. Under standard assumptions~\citep{bennett1948spectra}, we prove a  \textit{gamma-squared variance law} for symmetric grids, saying that they inflate output variance by a factor of $\gamma^2$ relative to asymmetric quantization. We are the first to formally connect the asymmetric-vs-symmetric design choice to token-level decision flips, showing that asymmetric quantization greatly improves task-level accuracy.

    \item We present \textit{SLQ}, a non-uniform scalar quantization pipeline using asymmetric quantization across bitwidths $\{2,3,4,5,6,7,8\}$, with multi-bitwidth Shapley sensitivity estimation and ILP-based allocation. For DL targets, the search operates directly on EAR constraints from calibration data. For TL targets, we introduce a \textit{single-point calibration} procedure that exploits the linear KL-recovery relationship to derive bitwidth requirements from one benchmark measurement. Because allocation follows a fidelity target, SLQ is \emph{Pareto-dominant} over uniform quantizers in both directions: it saves bits where a uniform grid overshoots (TL at \textbf{3.59} vs.\ 4 bpp, Llama-3.3-70B) and spends them where one falls short (QTIP misses TL, Llama-3.1-8B). 

    \item TL compression is achievable as low as \textbf{3.3} bpp, while DL requires \textbf{5--6 bpp}, corresponding to EAR $\geq 0.99$. Humming kernels~\citep{humming2025} serve all integer bitwidths efficiently on NVIDIA GPUs, and end-to-end vLLM serving yields $1.7$--$3.7\times$ throughput over BF16, consistently outperforming FP8.
\end{itemize}
\noindent\textbf{Related Work.}
Significant prior work focused on  \emph{weight-only} post-training quantization, from popular practical tooling (e.g., bitsandbytes~\citep{dettmers2022llmint8}) to algorithmic methods such as GPTQ~\citep{frantar2023gptq}, AWQ~\citep{lin2024awq}, SqueezeLLM~\citep{kim2024squeezellm}, SpQR~\citep{dettmers2024spqr}, OmniQuant~\citep{Shao2023OmniQuant}, HQQ~\citep{badri2023hqq}, QuIP and its follow-ups~\citep{chee2023quip,tseng2024quipsharp}, additive/vector-quantization approaches (AQLM)\citep{egiazarian2024aqlm, van2024gptvq}, as well as more recent trellis-based methods such as QTIP~\citep{tseng2024qtip}, and YAQA~\citep{tseng2025yaqa}. A growing line of work targets \emph{weight-and-activation} quantization via transformations and kernel support, e.g., QuaRot~\citep{ashkboos2024quarot} and SpinQuant~\citep{liu2024spinquant}.
Most techniques target fixed low-bit operating points (e.g., W4 or W4A4), and therefore primarily study the classic compression--accuracy trade-off rather than focusing on the compression achievable in the \emph{near-lossless} regime which is the focus of this paper. Among prior work, only SpQR and YAQA explicitly consider this frontier; nevertheless, SpQR falls well short of the strict accuracy targets we consider, while YAQA relies on substantially heavier system support (e.g., full-model-KL-aware Hessian sketching and integration with specialized quantizers/kernels such as QTIP). In addition, as we show experimentally, YAQA still falls short of the distribution lossless (DL) objective. Complementary to compression methods themselves, \citet{kubler2026degradations} address how to \emph{decide} whether an optimized model has degraded at all, proposing a McNemar-based hypothesis test that confronts per-sample outcomes rather than aggregated task scores, and thereby separates genuine degradation from evaluation noise. This is close in spirit to our task-lossless criterion: where they provide a test for detecting degradation in a given model, we turn the corresponding fidelity target into the objective that drives the search for the minimum bitwidth.

There has also recently been work on \textit{non-uniform precision allocation} for LLMs: for instance, HIGGS~\citep{malinovskii2025higgs} leverages linearity-style decompositions to push accuracy at low bitwidths;  EvoPress~\citep{sieberling2025evopress} applies evolutionary search for dynamic compression; OWL applies layer-wise non-uniformity via outlier-guided layer-wise sparsity choices~\citep{yin2023outlier}; and interaction-aware mixed-precision methods (e.g., IMPQ) incorporate cross-layer effects in allocation~\citep{zhao2025impq}. 
Our work is complementary: rather than optimizing within the standard lossy region, we specifically focus on the \emph{near-lossless} range where very small distributional shifts  matter. We show for the first time that reaching this regime requires (i) asymmetric quantization and (ii) searching across a \emph{wide} set of integer bitwidths  rather than a small set of precisions studied by prior work.

A separate line of work targets \emph{lossless} (bit-exact) compression of model weights. ZipNN~\citep{hershcovitch2025zipnn} exploits redundancy in weight representations to achieve lossless compression and efficient (de)compression. Yet, it does not provide inference-time acceleration since weights must be decoded prior to execution: their method is designed to reduce weight storage. \citet{yang2025compress} study exponent concentration and propose an entropy-coded \emph{FP8}-centric representation (ECF8) with GPU-optimized decoding, reporting up to 26.9\% memory savings and throughput improvements under fixed memory budgets; however, the approach is lossless with respect to an FP8 baseline and therefore applies most directly when models are stored/served in FP8. Finally, DFloat11~\citep{zhang2025dfloat11} introduces dynamic-length, lossless compression of BF16 weights with online GPU decompression for inference under tight memory constraints. These works operate strictly in the lossless regime; for this, they require higher effective bitrates than our work.

\section{Motivation}
\label{sec:motivation}

\begin{figure}[t!]
  \centering
  \begin{subfigure}[t]{0.49\linewidth}
    \centering
    \includegraphics[width=\linewidth]{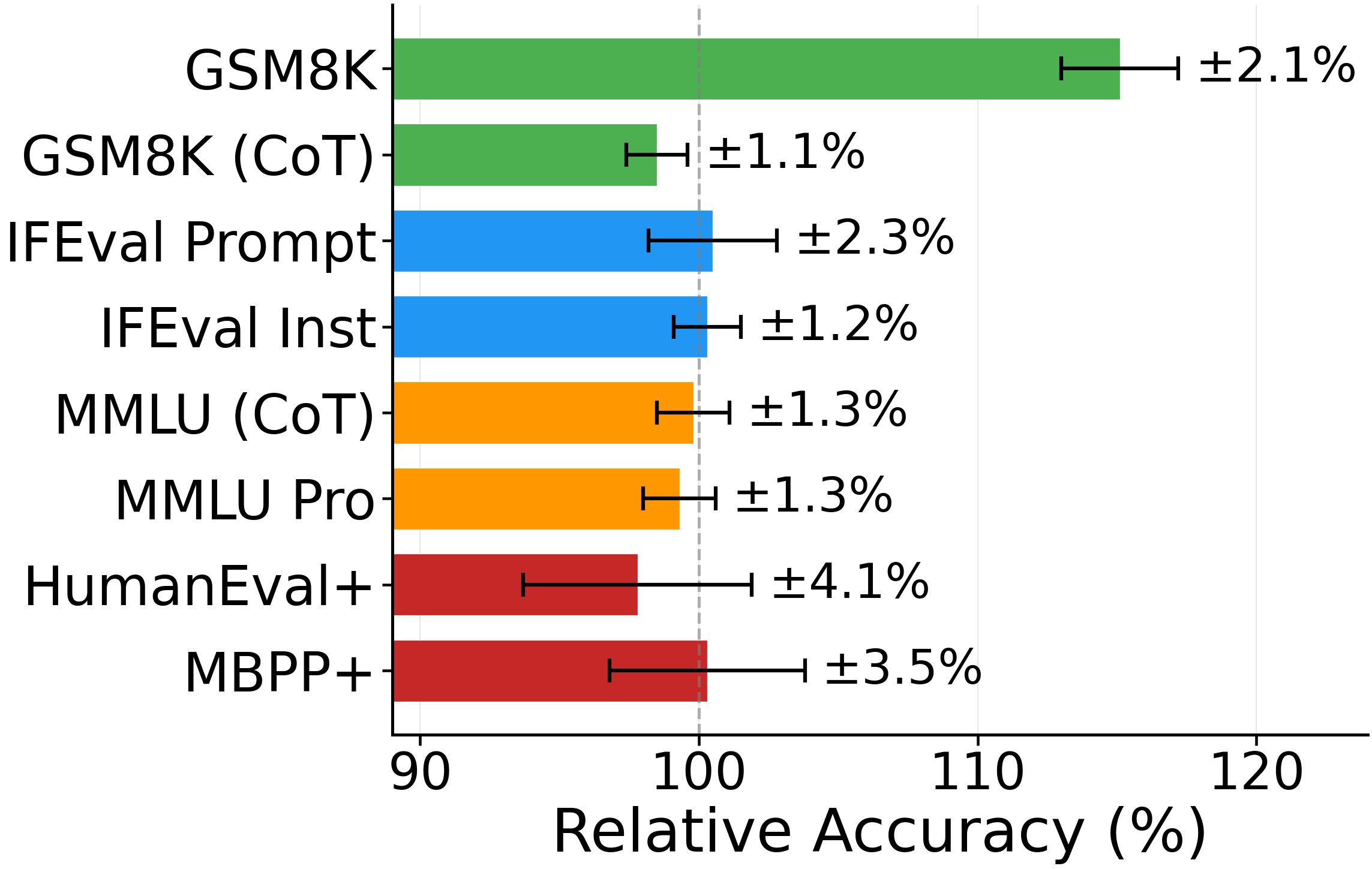}
    \caption{Sampling variance of \emph{uncompressed} Qwen3-8B.}
    \label{fig:uncompressed_variance}
  \end{subfigure}
  \hfill
  \begin{subfigure}[t]{0.49\linewidth}
    \centering
    \includegraphics[width=\linewidth]{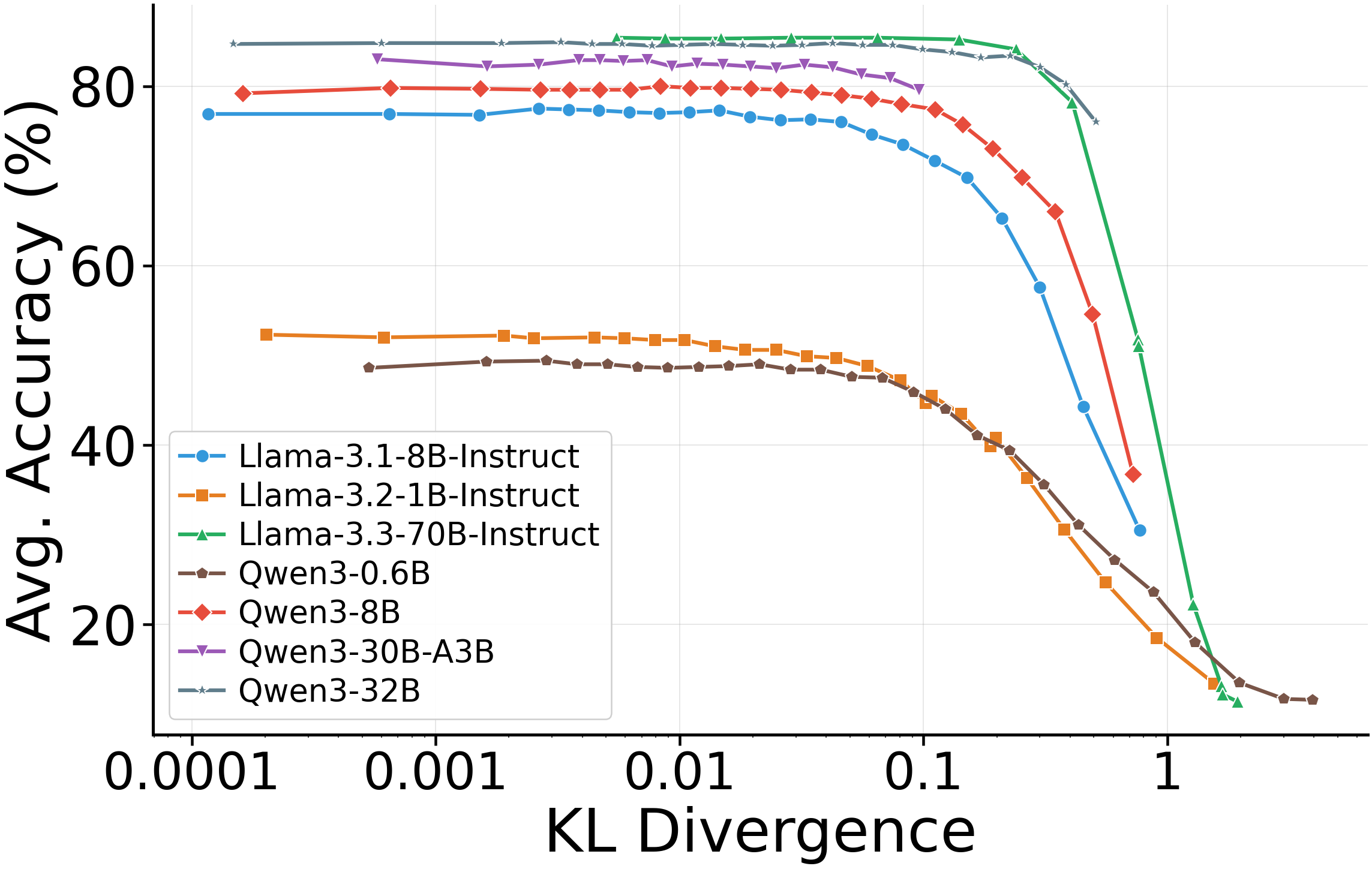}
    \caption{KL divergence vs.\ accuracy (log scale).}
    \label{fig:kl_acc_all}
  \end{subfigure}
  \caption{(a)~Run-to-run variance under stochastic sampling (dotted = greedy). (b)~Accuracy vs.\ KL divergence across seven models; a linear-scale zoom into the near-lossless regime is given in Figure~\ref{fig:kl_acc_linear}, and the comprehensive per-task analysis in Appendix~\ref{app:kl_acc}.}
  \label{fig:motivation}
\end{figure}

Most open-weight LLMs are deployed with \textit{stochastic decoding}, so even \emph{uncompressed} models exhibit measurable run-to-run variance on zero-shot evaluations (Figure~\ref{fig:motivation}, left). A quantized model that stays within this natural variance is, for all practical purposes, lossless.
Figure~\ref{fig:kl_acc_all} shows that in the near-lossless regime, zero-shot accuracy over standard tasks degrades approximately linearly with the KL divergence between the original and compressed models (Figure~\ref{fig:kl_acc_linear}, Appendix~\ref{app:kl_acc}, zooms in on a linear scale). This linearity means a single reference point is sufficient to anchor the compression search, without resorting to arbitrary KL thresholds; we formalize this as a single-point calibration procedure in Section~\ref{sec:solver}.

Uniform quantization around 4 bits per parameter (bpp) has become the standard target for memory-constrained LLM deployment~\citep{frantar2023gptq, lin2024awq, Kurtic2025BF16orDeath}. 
However, this target is arbitrary from an accuracy perspective: both GPTQ~\citep{frantar2023gptq} and QTIP~\citep{tseng2024qtip} incur significant KL divergence at ${\sim}$4\,bpp on Llama-3.1-8B-Instruct~\citep{dubey2024llama3}, while our non-uniform allocation achieves $\text{KL} = 0.004$ at EAR $0.994$ and 5.59\,bpp (Table~\ref{tab:kl-comparison}, Appendix).
More generally, different models and layers have different sensitivities, so any fixed uniform bitwidth is either too aggressive or wasteful. Our approach inverts the problem: given a fidelity target, find the minimum non-uniform bitwidth that achieves it. We validate this across models in Section~\ref{sec:accuracy}.

\section{Method}
\label{sec:methods}

We focus on obtaining near-lossless quantized models via \textit{layer-wise non-uniform scalar quantization}, chosen for its broad support across GPUs~\citep{frantar2024marlin, frantar2023gptq, lin2024awq} and CPUs~\citep{llama.cpp, pegolotti2023qigen, Ma2024BitNetb158}; the approach should be valid for more complex quantized representations (e.g., vector quantization) as well. 
Similarly, while we focus on GPTQ as our main quantization method, our approach could work with arbitrary other methods as well. 
Non-uniform allocation is motivated by the observation that different layers exhibit different sensitivities to quantization~\citep{dong2019hawq, kim2024squeezellm}. Given this target, we first show theoretically that asymmetric quantization is essential for distribution-lossless recovery (Section~\ref{sec:asym-theory}). We will then describe the SLQ pipeline: output metrics (Section~\ref{sec:formulation}), sensitivity estimation (Section~\ref{sec:sensitivity}), and bitwidth allocation (Section~\ref{sec:solver}).

\subsection{Asymmetric vs.\ Symmetric Quantization}
\label{sec:asym-theory}

\noindent\textbf{Setup.} Consider a weight tensor with values in $[L, U]$. Let $R = U - L$ denote the dynamic range, $M = \max(|L|, |U|)$ the maximum absolute magnitude, and $n$ the number of quantization levels. Symmetric quantization anchors the grid at zero, spanning $[-M, M]$ where $M = \max(|L|, |U|)$. Asymmetric quantization uses a zeropoint to shift the grid, spanning exactly $[L, U]$. Intuitively, symmetric quantization's zero-anchored grid wastes representational capacity when weight distributions are offset from zero, yielding larger per-layer output errors that increase the probability of argmax changes. We provide an overview of the formalization below, showing the impact is \emph{quadratic} in the centering inefficiency. The full argument is in Appendix~\ref{app:bussgang}.

\begin{definition}[Centering inefficiency]
\label{def:centering-inefficiency}
The centering inefficiency is $\gamma = 2M/R$.
\end{definition}

When weights are symmetric about zero, $\gamma = 1$. For asymmetric distributions (typical in LLMs), $\gamma > 1$. For example, weights in $[-0.8, 1.2]$ yield $\gamma = 1.2$: the symmetric grid spans $[-1.2, 1.2]$, wasting 20\% of its capacity on the interval $[-1.2, -0.8]$ where no weights exist.

\noindent\textbf{From step sizes to noise variance.} Symmetric quantization has step size $\Delta_{\text{sym}} = 2M/(n{-}1)$, while asymmetric has $\Delta_{\text{asym}} = R/(n{-}1)$. Their ratio is $\Delta_{\text{sym}} / \Delta_{\text{asym}} = 2M/R = \gamma$, independent of the number of levels. Under the high-rate approximation~\citep{bennett1948spectra}, quantization error is approximately uniform within each bin, with variance $\sigma^2 = \Delta^2/12$. Since $\Delta_{\text{sym}} = \gamma \cdot \Delta_{\text{asym}}$, the variance inherits a \emph{quadratic} dependence:

\begin{lemma}[Gamma-Squared Variance Law]
\label{thm:gamma-squared}
The quantization noise variance under symmetric quantization scales quadratically with the centering inefficiency relative to asymmetric quantization:
\begin{equation}
    \sigma_{\mathrm{sym}}^2 = \gamma^2 \, \sigma_{\mathrm{asym}}^2.
\end{equation}
\end{lemma}

\noindent\textbf{Impact on output error.} For a single linear layer $\mathbf{y} = W\mathbf{x}$, quantization produces output perturbation $\boldsymbol{\delta} = (Q(W) - W)\mathbf{x}$, with expected squared norm $\sigma^2 \|\mathbf{x}\|^2$. Since quantization is applied to identical inputs, $\|\mathbf{x}\|^2$ cancels out:
\begin{equation}
    \frac{\mathbb{E}[\|\boldsymbol{\delta}_{\mathrm{sym}}\|^2]}{\mathbb{E}[\|\boldsymbol{\delta}_{\mathrm{asym}}\|^2]} = \gamma^2.
\end{equation}
This implies that a distribution with $\gamma = 1.2$, which is only mildly asymmetric, suffers 44\% more per-layer output error under symmetric quantization. This compounds across layers and increases the probability of \emph{decision flips}: positions where quantization perturbs logits enough to change the argmax token, directly reducing EAR. The full derivation is in Appendix~\ref{app:bussgang}.

\noindent\textbf{Empirical Validation.}
\label{empirical:asym}
We now validate the $\gamma^2$ scaling on Llama-3.1-8B-Instruct (uniform 4-bit). Figures~\ref{fig:flip-by-entropy-absolute}--\ref{fig:flip-by-entropy-ratio} show symmetric quantization flips more tokens across all entropy bins, with the ratio largest at low entropy (up to $3\times$). The resulting EAR gap is clearest at the DL threshold itself: under Shapley allocation over the full bitwidth range, asymmetric quantization reaches $\text{EAR} \geq 0.99$ at 5.5\,bpp, whereas symmetric does not reach it until 7.0\,bpp, a gap of 1.5\,bpp (Appendix~\ref{app:detailed-results}). Figure~\ref{fig:layer-acceptance} confirms that the effect compounds across layers: symmetric quantization degrades more sharply in the first few sensitive layers, undermining mixed-precision schemes. On MMLU-Pro, where each answer is a single-token decision, asymmetric quantization recovers 88.4\% of FP16 accuracy versus 82.5\% for symmetric (Appendix~\ref{app:asym-empirical}).

\noindent
\begin{minipage}{0.32\linewidth}
  \centering
  \includegraphics[width=\linewidth]{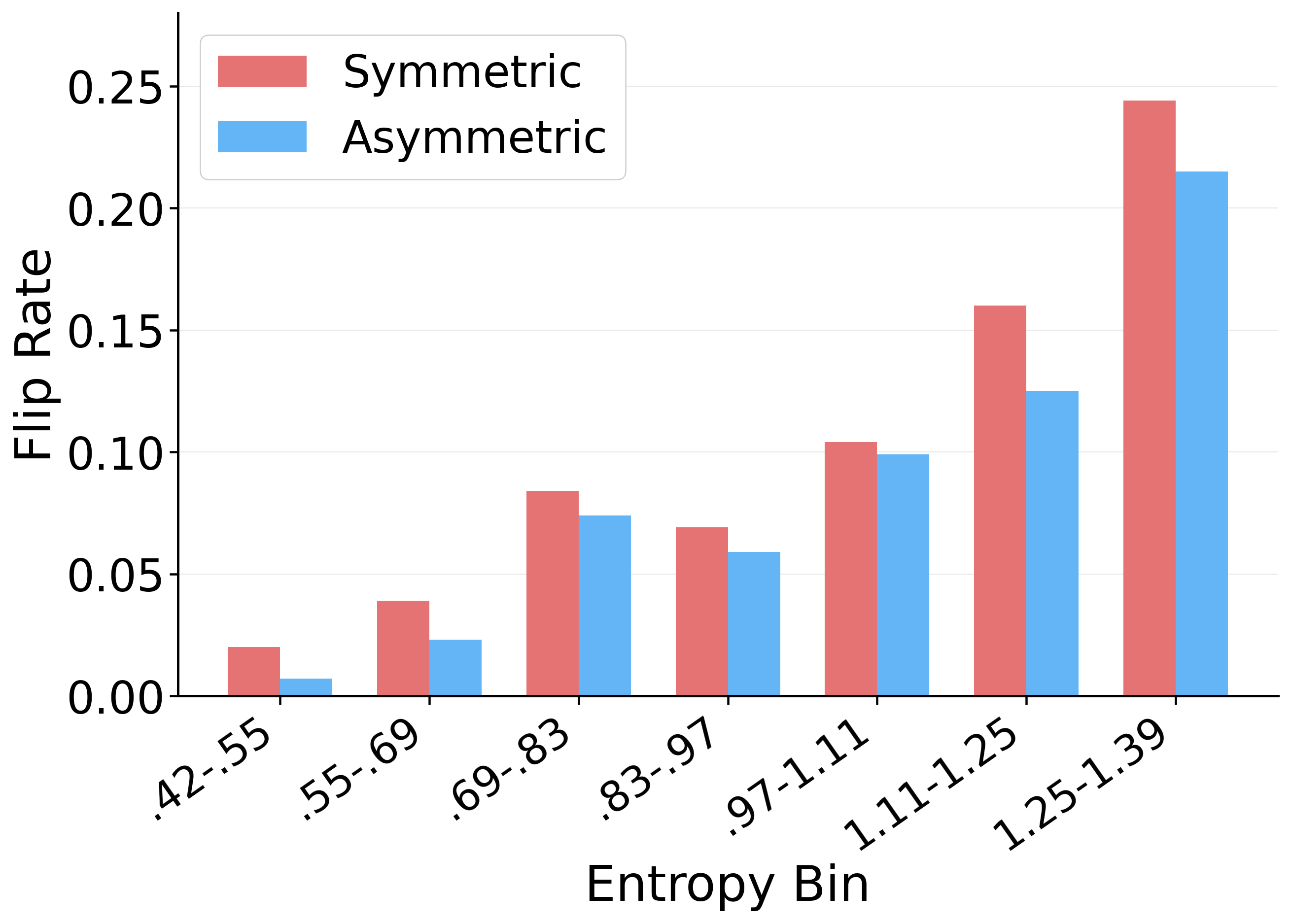}
  \captionof{figure}{Decision flip rates.}
  \label{fig:flip-by-entropy-absolute}
\end{minipage}
\hfill
\begin{minipage}{0.32\linewidth}
  \centering
  \includegraphics[width=\linewidth]{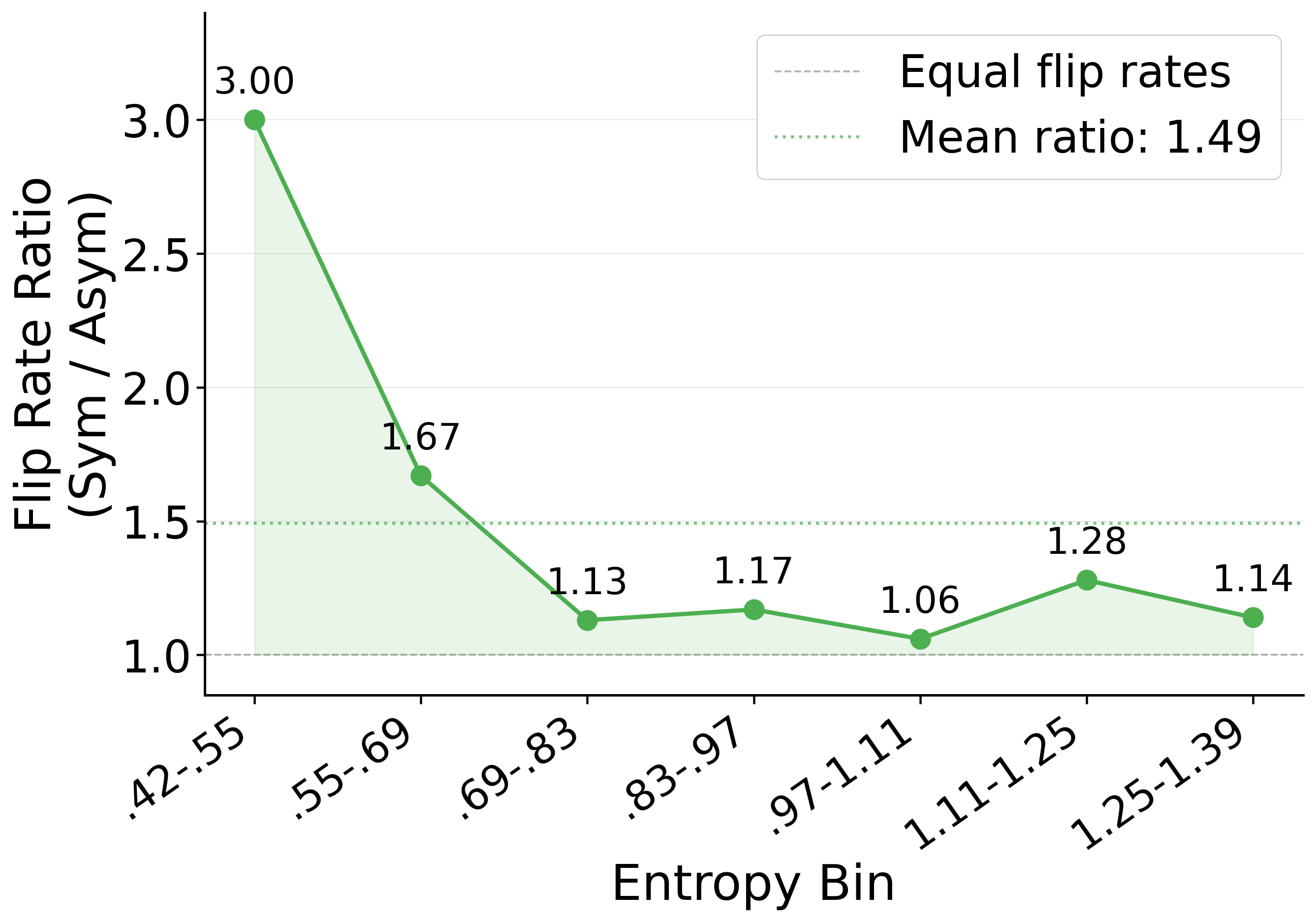}
  \captionof{figure}{Sym-to-asym ratio.}
  \label{fig:flip-by-entropy-ratio}
\end{minipage}
\hfill
\begin{minipage}{0.32\linewidth}
  \centering
  \includegraphics[width=\linewidth]{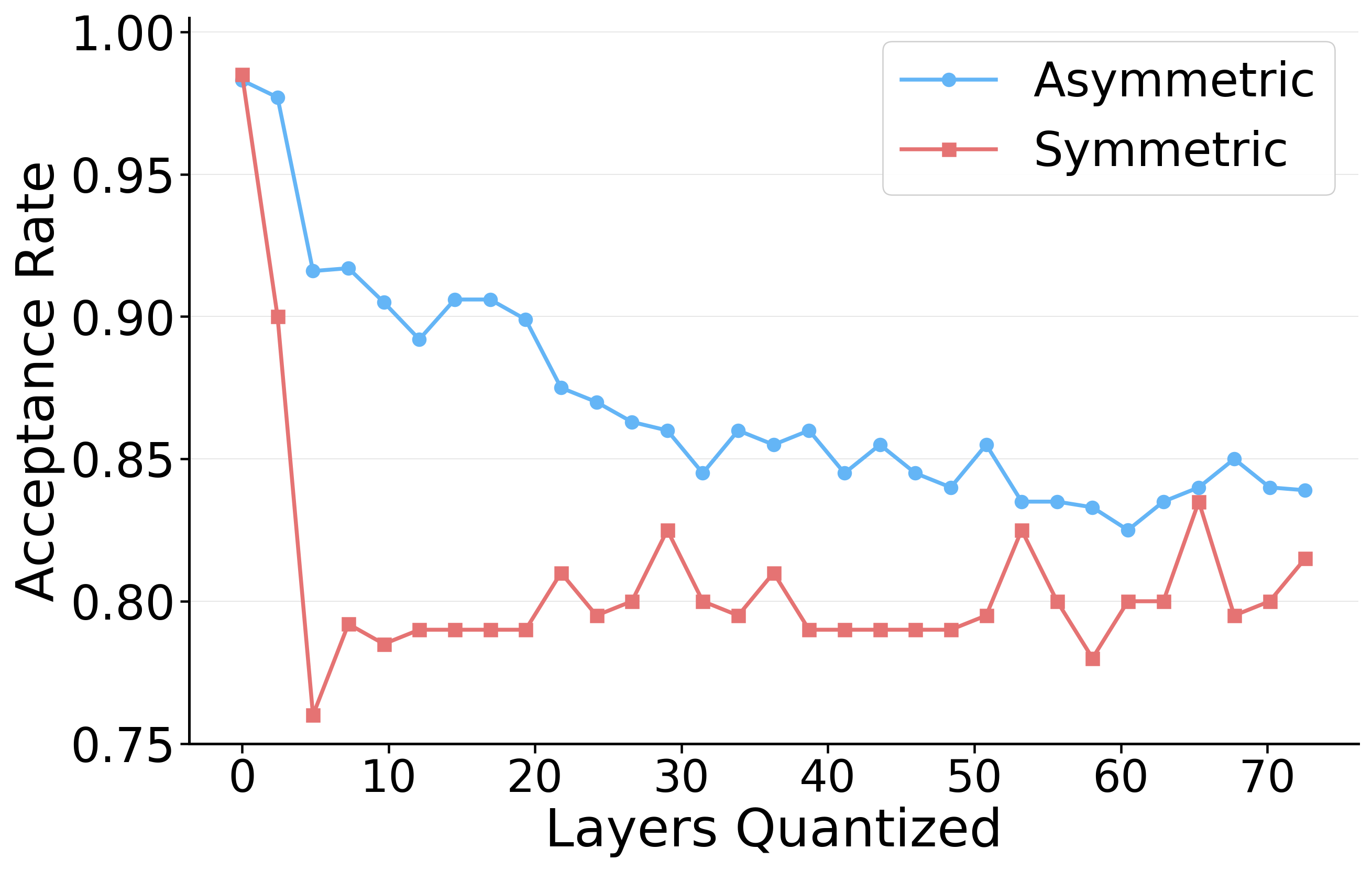}
  \captionof{figure}{Layer-by-layer.}
  \label{fig:layer-acceptance}
\end{minipage}

\vspace{0.5em}

These experiments show that the asymmetric--symmetric gap is most pronounced at the distributional level: achieving EAR $\geq 0.99$ with symmetric quantization requires up to 1.5 additional bits per parameter. For task-level accuracy, symmetric quantization preserves high average recovery, because stochastic sampling absorbs the additional token flips without measurably affecting aggregate scores. We conclude that asymmetric  quantization is therefore a prerequisite for distribution-lossless compression, but symmetric remains viable when only task-level preservation is required. 

\noindent\textbf{Relation to prior work.} That asymmetric grids reduce reconstruction error is known. The consequence at the token level has not been made explicit. The $\gamma^2$ inflation causes symmetric quantization to flip tokens that asymmetric leaves intact, and the flip-rate ratio is widest in the low-entropy bins, where the original model was most confident (Figure~\ref{fig:flip-by-entropy-ratio}). The grid geometry therefore determines whether a given bit budget can reach the DL threshold.

\subsection{Output Metrics for Losslessness}
\label{sec:formulation}

The non-uniform quantization problem is that of assigning  a bitwidth $b_\ell \in \mathcal{B} = \{2,3,4,5,6,7,8\}$ to each layer $\ell$. The average bitwidth is $\bar{b}(\mathbf{b}) = \sum_{\ell} b_\ell |W_\ell| / \sum_{\ell} |W_\ell|$. The first challenge is that of precise output metrics, computable from a small calibration set, that reliably predict when a configuration is near-lossless.

\noindent\textbf{KL-Divergence.} For $N$ calibration token positions with original and quantized distributions $p_i$, $q_i$, we compute $D_{\mathrm{KL}} = \frac{1}{N} \sum_{i} \sum_{k \in \mathcal{T}_i} p_i(k) \log \frac{p_i(k)}{q_i(k)}$, restricted to the top-$K{=}10$ tokens under $p_i$. However, KL divergence is measured in natural log units (nats), unbounded above, and asymmetric. Its magnitude is difficult to interpret without calibrated intuition for a given vocabulary and distribution shape. While Pinsker's inequality ($d_{TV} \leq \sqrt{D_{\mathrm{KL}}/2}$) bounds KL into a percentage-like quantity, this bound is loose.

\noindent\textbf{Expected Acceptance Rate (EAR).} We therefore define the \textit{Expected Acceptance Rate (EAR)} as the primary metric for \emph{distribution-lossless} compression, inspired by the token acceptance probability in speculative decoding~\citep{leviathan2023fast}. In that setting, a draft token sampled from a proposal distribution $q$ is accepted under the target distribution $p$ with probability $\sum_k \min(p(k), q(k))$; the same quantity measures the fidelity of a quantized model's output distribution to the original. Formally, this probability mass overlap equals $1 - d_{TV}(p,q)$, is bounded in $[0, 1]$, and is directly interpretable as a percentage of shared probability mass. By the coupling characterization of total variation distance, $\sum_k \min(p(k), q(k))$ equals the maximum probability that two random variables $X \sim p$ and $Y \sim q$ can be coupled to agree ($X = Y$). We average this over $N$ calibration positions:
\begin{equation}
    \mathrm{EAR} = \frac{1}{N} \sum_{i} \sum_{k=1}^{K} \min(p_i(k), q_i(k)),
\end{equation}
where both distributions are restricted to the top-$K$ tokens under $p_i$. EAR has a direct operational interpretation: $\text{EAR} = 0.99$ means that, under optimal coupling, the original and quantized models produce the same token 99\% of the time. While the underlying quantity $1 - d_{TV}$ is classical, its application as a compression fidelity criterion for quantized LLMs is, to our knowledge, new.

Restricting to top-$K$ tokens is a tight approximation: LLM output distributions are sharply peaked, so the excluded tail contributes negligible mass. EAR is typically more restrictive than KL ($\text{EAR} \geq 0.99$ generally implies $D_{\mathrm{KL}} \ll 0.01$), and is robust to tail behavior since $\min(p, q)$ clips near-zero probabilities rather than producing unbounded penalties. We use EAR as the primary constraint for bitwidth search (Section~\ref{sec:solver}), with KL as a secondary diagnostic. Both are computed simultaneously during each forward pass at no additional cost.


\subsection{Sensitivity Estimation and Bitwidth Allocation}
\label{sec:sensitivity}

Given the output metrics, we need to assign each layer a bitwidth from the budget $\mathcal{B} = \{2, 3, 4,5,6,7,8\}$. Prior mixed-precision work typically uses binary allocation (e.g., a layer is set to either 4 or 8 bits)~\citep{zhao2025impq, dong2019hawq}. We find this experimentally to be too coarse for the near-lossless regime: restricting to $\{4, 8\}$ consistently requires higher average bitwidth than the full range to meet the same quality constraint (see Section~\ref{sec:accuracy}). Since inference engines often require certain layers to share a bitwidth (e.g., fused QKV projections in vLLM), we partition layers into $M$ groups $\mathcal{G} = \{G_1, \ldots, G_M\}$.

\noindent\textbf{Sensitivity Estimation.}
\label{sec:sensitivity-estimation}
We need to estimate how each group's quantization at a given bitwidth affects the output metrics. A simple baseline is \emph{linear estimation}~\citep{malinovskii2025higgs}: model metric degradation as $\Delta_\text{KL}(\mathbf{b}) \approx \sum_{m} \alpha_m \cdot e_m^{(b_m)}$, where $e_m^{(b)}$ is the normalized reconstruction error and $\alpha_m$ is estimated via noise injection ($O(T \cdot M)$ forward passes). This is fast, but assumes independence across groups.

\noindent\textbf{Multi-Bitwidth Shapley Estimation.}
\label{sec:shapley}
Shapley-based sensitivity estimation~\citep{shapley1953value} has been applied to mixed-precision quantization~\citep{zhao2025impq}, pruning~\mbox{\citep{adamczewski2024shapleyoracle, sun2025shapley}}, and layer importance analysis~\mbox{\citep{zhang2024investigating}}. However, existing approaches consider layer-wise binary choices, switching layers between one high and one low precision, which restricts allocation to a binary choice and would thus lead to higher average bitwidths (Section~\ref{sec:accuracy} and Appendix~\ref{app:detailed-results}). 

We extend this to support the full bitwidth range by running a \emph{separate} binary game for each target bitwidth $b^* \in \mathcal{B} \setminus \{b_\text{max}\}$. In each game, we sample $P$ random permutations, start with all groups at $b_\text{max}$, and sequentially switch each group to $b^*$, recording the marginal metric change. The Shapley value $\phi_m^{(b^*)}$ for group $G_m$ is the average marginal contribution across permutations. Across all $|\mathcal{B}|-1$ games, this yields a full sensitivity database $\phi_m^{(b)}$ for every group-bitwidth pair. The games are independent and can run in parallel; each forward pass computes KL and EAR simultaneously. The total cost is $O(P \cdot M \cdot |\mathcal{B}|)$ forward passes. The full procedure is given in Algorithm~\ref{alg:multi-shapley} (Appendix~\ref{app:shapley-algorithm}).

\noindent\textbf{Bitwidth Allocation.}
\label{sec:solver}
Given the sensitivity database $c_{m,b}$ from either estimation method, we solve a multiple-choice knapsack problem via integer linear programming (ILP): assign each group $G_m$ a bitwidth $b$ to minimize total predicted degradation $\sum_{m,b} x_{m,b} \cdot c_{m,b}$ subject to a bitwidth budget. We then binary-search over the budget to find the minimum average bitwidth satisfying a quality constraint. The two compression targets use different constraints, described below. The same framework extends to joint weight-and-activation quantization by quantizing activations in the sensitivity estimates and allocation (Appendix~\ref{app:detailed-results}). The sensitivity database is computed once, and each ILP solve then takes milliseconds, so the binary search adds little cost. SLQ is in principle solver-agnostic: more accurate alternatives such as gradient-based search~\citep{helcig2026rco} are slower, but remain fast enough to be practical. Details on metric prediction and an alternative evolutionary search appear in Appendices~\ref{app:metric_prediction} and~\ref{app:evo}. Appendix~\ref{app:sensitivity-plots} visualizes the per-layer sensitivity landscape and resulting bitwidth assignments.

\noindent\textbf{Distribution-lossless search.} The sensitivity database directly predicts EAR and KL for any candidate configuration (Appendix~\ref{app:metric_prediction}). We set a target such as $\text{EAR} \geq 0.99$ and binary-search over the bitwidth budget until the predicted EAR meets the constraint. The entire search operates on calibration-set metrics and requires no downstream benchmark evaluation, and is stable across calibration seeds: repeated runs agree to within ${\pm}0.0005$ EAR.

\noindent\textbf{Task-lossless search via single-point calibration.} For task-lossless compression, the goal is to find the minimum bitwidth that preserves a target recovery rate (e.g., ${\geq}\,99\%$ of BF16 accuracy) across all benchmarks. We exploit the approximate linearity between KL divergence and task recovery established in Section~\ref{sec:motivation} (Figure~\ref{fig:kl_acc_all}):
\begin{equation}
    \text{recovery} \approx 1 - \alpha \cdot D_{\mathrm{KL}}.
    \label{eq:linear-recovery}
\end{equation}
This linear regime holds reliably for the high-bitwidth configurations relevant to near-lossless compression, where the perturbation is small enough that accuracy degrades smoothly with distributional shift. We observe this consistently across six model families and five task categories (Appendix~\ref{app:kl_acc}): within the near-lossless range, recovery is well-approximated by a single slope, while at more aggressive compression the relationship becomes non-linear and task-dependent.

The intercept is known ($D_{\mathrm{KL}} {=} 0$ implies $\text{recovery} {=} 1$), so a single calibration point determines the slope~$\alpha$. We choose a uniform bitwidth where the model shows meaningful degradation (e.g., uniform 4-bit), measure its actual KL divergence on calibration data, and evaluate benchmarks to obtain recovery. This yields $\alpha = (1 - \text{recovery}) / D_{\mathrm{KL}}^{\text{cal}}$, and inverting gives the KL threshold for any target recovery: $D_{\mathrm{KL}}^{\text{thresh}} = (1 - \text{target recovery}) / \alpha$.

Since Shapley-predicted KL may differ from measured KL by a systematic factor, we compute a calibration ratio $\rho = D_{\mathrm{KL}}^{\text{actual}} / D_{\mathrm{KL}}^{\text{predicted}}$ at the calibration point and apply it to all subsequent predictions: $\hat{D}_{\mathrm{KL}} = \rho \cdot D_{\mathrm{KL}}^{\text{predicted}}$. The binary search then checks $\hat{D}_{\mathrm{KL}} \leq D_{\mathrm{KL}}^{\text{thresh}}$ at each candidate bitwidth, requiring no additional forward passes. The total cost is one forward pass to measure KL plus two benchmark runs (BF16 baseline and calibration configuration); the search itself is arithmetic on the sensitivity database. The full procedure is given in Algorithm~\ref{alg:tl-search} (Appendix~\ref{app:tl-calibration}).

\noindent\textbf{Guardrail against out-of-regime anchors.} Since Equation~\ref{eq:linear-recovery} and $\rho$ are fitted from a single point, an anchor outside the near-lossless regime would extrapolate a poor threshold. We therefore add a held-out check: once the search converges we measure the actual KL and reject the configuration if the actual-to-predicted ratio diverges from $\rho$ by more than $2\times$, re-anchoring elsewhere. This costs one forward pass; within the regime $\alpha$ is nearly constant, so 4-bit and 6-bit anchors agree, while 3-bit is not recommended.

\section{Experiments}
\label{sec:experiments}

We evaluate SLQ for weight-only quantization on four models: Qwen3-8B, Qwen3-32B, Qwen3.5-27B~\citep{bai2023qwen}, and Llama-3.3-70B-Instruct~\citep{dubey2024llama3}, and on the Qwen3.5-35B-A3B mixture-of-experts model (Appendix~\ref{app:moe}). 
We quantize layers using the standard GPTQ technique.
Table~\ref{tab:main-accuracy} reports per-benchmark scores and average recovery rate (Avg Rec.) relative to the BF16 baseline under non-greedy decoding, averaged over 3 seeds; values of 1.0 indicate exact baseline recovery. Each model block shows SLQ-DL and SLQ-TL configurations alongside uniform W4A16 or GPTQ-Int4 references. We report average bits per parameter, accounting for all quantization overheads (Appendix~\ref{app:actual_bitwidths}).

\noindent\textbf{Setup.} Following our prior analysis, all models use asymmetric integer group quantization (group size 128, bitwidths $\{2,3,4,5,6,7,8\}$) with Shapley-based per-layer allocation (Section~\ref{sec:sensitivity}) using 512 calibration samples; DL targets use $\text{EAR} \geq 0.985$--$0.99$ depending on the model. The uniform W4A16 baselines use GPTQ~\citep{frantar2023gptq} with tuned hyperparameters; SLQ is orthogonal to the choice of per-layer quantizer and instead provides a principled method for finding the minimum-bitwidth non-uniform configuration that meets a given fidelity target. Quantized weights are served with Humming kernels~\citep{humming2025} in vLLM~\citep{kwon2023vllm} (MARLIN-based~\citep{frantar2024marlin}, 4--8 bit; Appendix~\ref{app:speedups}). Zero-shot accuracy is measured on seven benchmarks: reasoning (ARC-Challenge, GSM8K, GSM8K-CoT), instruction following (IFEval), and knowledge (MMLU-CoT, MMLU-Pro), under non-greedy decoding with model-recommended sampling parameters, averaged over 3 seeds. End-to-end throughput is measured on NVIDIA L40s GPUs, running the Humming kernels in vLLM, using ShareGPT and Reasoning workloads; we report output tokens per second per GPU (TPS/GPU) to normalize across tensor-parallelism configurations.

\subsection{Accuracy Results}
\label{sec:accuracy}

\begin{table*}[t!]
\centering
\caption{Non-greedy accuracy and recovery rates. W4A16 shown for reference. Bold only for visibility. GPTQ-Int4 is the official 4-bit quantized Qwen3.5-27B. \protect\footnotemark}
\label{tab:main-accuracy}
\small
\resizebox{\textwidth}{!}{%
\begin{tabular}{@{}lccccccccc@{}}
\toprule
Model & Bits & ARC & GSM8K & GSM8K & IFEval & IFEval & MMLU & MMLU & Avg \\
 & & Chal. & & (CoT) & Prompt & Inst. & CoT & Pro & Rec. \\
\midrule
\multicolumn{10}{l}{\textit{Qwen3-8B} (non-greedy, temp=0.6, top\_p=0.95, top\_k=20, 3 seeds)} \\
\midrule
Baseline (BF16) & 16.0 & 90.64 & 79.15 & 90.75 & 82.44 & 88.13 & 71.23 & 63.42 & 1.000 \\
SLQ-DL & 5.70 & 90.30 & 80.44 & 90.45 & 80.96 & 87.17 & 68.95 & 63.45 & 0.993 \\
SLQ-TL & 4.72 & 90.13 & 80.43 & 90.60 & 82.99 & 88.49 & 71.86 & 62.92 & 1.003 \\
Uniform-4 & 4.16 & 89.27 & 80.74 & 90.22 & 82.26 & 87.77 & 60.01 & 61.80 & 0.976 \\
\midrule
\multicolumn{10}{l}{\textit{Qwen3-32B} (non-greedy, temp=0.6, top\_p=0.95, top\_k=20, 3 seeds)} \\
\midrule
Baseline (BF16) & 16.0 & 92.50 & 91.72 & 91.73 & 86.40 & 89.22 & 86.30 & 70.70 & 1.000 \\
SLQ-DL & 4.98 & 92.28 & 93.27 & 91.48 & 86.53 & 89.82 & 87.32 & 70.12 & 1.003 \\
SLQ-TL & 3.59 & 92.30 & 93.33 & 91.22 & 85.89 & 88.90 & 86.01 & 68.92 & 0.996 \\
Uniform-4 & 4.16 & 91.88 & 91.39 & 90.90 & 86.01 & 89.12 & 85.44 & 68.90 & 0.991 \\
\midrule
\multicolumn{10}{l}{\textit{Qwen3.5-27B} (non-greedy, temp=1.0, top\_p=0.95, top\_k=20, 3 seeds)} \\
\midrule
Baseline (BF16) & 16.0 & 96.88 & 96.84 & 95.73 & 86.08 & 90.49 & 88.16 & 77.87 & 1.000 \\
SLQ-DL & 5.19 & 95.28 & 96.78 & 95.45 & 85.77 & 90.17 & 88.47 & 78.43 & 0.997 \\
SLQ-TL & 3.30 & 96.22 & 95.75 & 94.84 & 85.21 & 89.81 & 87.53 & 76.64 & 0.991 \\
GPTQ-Int4 & 4.1 & 96.65 & 96.66 & 95.56 & 86.41 & 90.65 & 88.23 & 77.68 & 0.999 \\
\midrule
\multicolumn{10}{l}{\textit{Llama-3.3-70B-Instruct} (non-greedy, temp=0.6, top\_p=0.9, 3 seeds)} \\
\midrule
Baseline (BF16) & 16.0 & 93.45 & 93.38 & 94.52 & 88.78 & 92.37 & 86.39 & 70.78 & 1.000 \\
SLQ-DL & 6.55 & 93.48 & 93.55 & 94.80 & 89.52 & 92.85 & 86.26 & 70.52 & 0.999 \\
SLQ-TL & 3.59 & 92.70 & 93.33 & 93.48 & 89.28 & 92.69 & 85.19 & 68.25 & 0.991 \\
Uniform-4 & 4.16 & 93.65 & 91.10 & 93.71 & 89.34 & 92.57 & 85.94 & 70.08 & 0.998 \\
\bottomrule
\end{tabular}}%
\end{table*}
\footnotetext{Bits averaged over subset (MLP + linear attention input/output projections); remaining layers are BF16. Following Qwen's convention; \href{https://huggingface.co/Qwen/Qwen3.5-27B-GPTQ-Int4}{Qwen3.5-27B-GPTQ-Int4} labeled \textit{Int4} under the same scheme.}

\noindent\textbf{Zero-shot accuracy as a fidelity metric.} In the near-lossless regime, zero-shot benchmarks are a blunt instrument: models with meaningfully different distributional fidelity can score within each other's noise bands (Figure~\ref{fig:uncompressed_variance}). This is why we use EAR/KL as the primary search objective for DL and a recovery threshold for TL, rather than optimizing benchmark scores directly.

\noindent\textbf{Distribution-lossless results.} The SLQ-DL rows show configurations optimized for distributional fidelity (EAR $\geq$ $0.99$). All four models achieve ${\geq}\,99.3\%$ average recovery at 4.98--6.55 bits, within the sampling variance of the baseline on every benchmark. Qwen3-32B compresses most aggressively (4.98 bits, 100.1\% recovery), while Llama-3.3-70B requires 6.55 bits.

\noindent\textbf{Task-lossless at aggressive bitwidths.}
When the goal is only task-level preservation, the single-point calibration procedure (Section~\ref{sec:solver}) allows substantially more aggressive bitwidths. The SLQ-TL configurations target ${\geq}\,99\%$ average recovery and achieve it across all four models: Qwen3.5-27B at 3.30 bits, Llama-3.3-70B at 3.59 bits, Qwen3-8B at 4.72 bits, and Qwen3-32B at 3.59 bits. In every case, SLQ-TL dominates uniform W4A16 at comparable or lower bitwidths.

\noindent\textbf{Comparison to state-of-the-art quantizers.}
QTIP~\citep{tseng2024qtip} and YAQA~\citep{tseng2025yaqa} are orthogonal to SLQ, so the question is whether uniform allocation with a stronger quantizer already reaches the statistically-lossless regime. Table~\ref{tab:sota-comparison} shows it does not. On Llama-3.1-8B, QTIP misses both targets at every bitwidth tested. On Llama-3.3-70B, QTIP and YAQA at 4 bits clear the task-lossless bar, but their EAR leaves them far from the $0.99$ required for distribution-losslessness. SLQ reaches TL on that model at 3.59\,bpp and DL at 6.55\,bpp (Table~\ref{tab:main-accuracy}), and on Llama-3.1-8B attains $\text{KL} = 0.004$ at 5.59\,bpp against QTIP's $0.017$ at 4 bits. A uniform budget is set without reference to any fidelity target, so it lands on the wrong side of the boundary in both directions.

\begin{table}[t]
\centering
\caption{QTIP and YAQA evaluated by us on the released checkpoints, under the benchmark suite and sampling setup of Table~\ref{tab:main-accuracy}. Task-losslessness requires ${\geq}\,99\%$ mean recovery, distribution-losslessness EAR $\geq 0.99$; no configuration reaches the latter.}
\label{tab:sota-comparison}
\small
\begin{tabular}{@{}llcccc@{}}
\toprule
Method & Model & Bits & Task mean rec. & KL $\downarrow$ & EAR $\uparrow$ \\
\midrule
QTIP & Llama-3.1-8B  & 4-bit & 97.94\% & 0.0172 & 0.9519 \\
QTIP & Llama-3.1-8B  & 3-bit & 96.00\% & 0.0510 & 0.9165 \\
QTIP & Llama-3.3-70B & 4-bit & 99.9\%  & 0.0244 & 0.9544 \\
YAQA & Llama-3.3-70B & 4-bit & 99.5\%  & 0.0150 & 0.9622 \\
\bottomrule
\end{tabular}
\end{table}

\noindent\textbf{Comparison to non-uniform allocators.}
HIGGS~\citep{malinovskii2025higgs}, EvoPress~\citep{sieberling2025evopress} and IMPQ~\citep{zhao2025impq} maximize quality under a fixed budget, and so cannot \emph{locate} the losslessness threshold without sweeping the entire curve. Against IMPQ on Qwen3-8B, task recovery saturates for both above ${\sim}4$\,bpp while EAR separates them (Appendix~\ref{app:allocators}).

\noindent\textbf{Ablations.} The full bitwidth range $\{4,5,6,7,8\}$ is essential: on Qwen3-32B at EAR $\geq 0.99$, restricting to $\{4, 8\}$ requires 5.74 bits while the full range achieves equivalent fidelity at 4.98 bits. Shapley-based allocation outperforms linear estimation at mid-range bitwidths (5.5--6.5 bits) where decisions are tightest. Comprehensive ablations and bitwidth-range comparisons appear in Appendix~\ref{app:detailed-results}.

\noindent\textbf{Token-Level Analysis for DL.}
We perform a teacher-forced token-level comparison between base Qwen3-8B and its 5.70-bit SLQ-DL quantization: responses generated by the base model (greedy) are fed through the quantized model to record per-position top-1 disagreements, isolating quantization effects from autoregressive divergence. Across 15 diverse prompts (6{,}531 tokens), 98.1\% of tokens are identical; the 1.9\% that diverge are minor stylistic variations (synonym swaps, punctuation) with no factual errors, distributed uniformly across positions. Per-prompt statistics and annotated examples are in Appendix~\ref{app:token-divergence}.

\noindent\textbf{Thinking token inflation.}
Thinking tokens dominate reasoning model inference cost~\citep{chen2026price}. On Qwen3-8B GSM8K (1319 problems, 3 seeds), SLQ-DL introduces \textit{no} additional thinking overhead relative to BF16, whereas SLQ-TL and Uniform-4 increase thinking tokens by $+2.2\%$ and $+2.9\%$, respectively (see Appendix~\ref{app:thinking-tokens} for full details). The gap is statistically clean (paired Wilcoxon: $p = 0.86$ for SLQ-DL, $p < 10^{-5}$ for both lossy configurations), so TL and DL are separate regimes rather than points on one axis: a model can be task-lossless and still cost measurably more to serve.

\subsection{Inference Speedup}
\label{sec:speedup-results}

Table~\ref{tab:e2e-speedup} presents end-to-end throughput on L40S GPUs, running the Humming kernels~\citep{humming2025} in vLLM. The SLQ configurations provide $1.7$--$3.7\times$ throughput over BF16 at ${\geq}\,99.1\%$ recovery, consistently outperforming FP8 per GPU ($1.23\times$ on Qwen3-8B, $1.54\times$ on Qwen3-32B). Speedups come from both faster kernel execution and reduced GPU count: Qwen3-32B fits on 1 GPU instead of 2, Llama-3.3-70B on 2 instead of 4.

\begin{table}[H]
\centering
\caption{End-to-end throughput on L40s GPUs (Humming kernels in vLLM). TPS/GPU normalizes by GPU count.}
\label{tab:e2e-speedup}
\small
\begin{tabular}{@{}lcccccc@{}}
\toprule
Model & Bits & GPUs & \multicolumn{2}{c}{ShareGPT} & \multicolumn{2}{c}{Reasoning} \\
 & & & TPS/GPU & Speedup & TPS/GPU & Speedup \\
\midrule
\multicolumn{7}{l}{\textit{Qwen3-8B}} \\
BF16 & 16.0 & 1 & 630.9 & 1.00$\times$ & 544.6 & 1.00$\times$ \\
\textbf{SLQ-DL} & \textbf{5.70} & \textbf{1} & \textbf{1097.2} & \textbf{1.74$\times$} & \textbf{907.6} & \textbf{1.67$\times$} \\
\textbf{SLQ-TL} & \textbf{4.72} & \textbf{1} & \textbf{1223.9} & \textbf{1.94$\times$} & \textbf{1010.9} & \textbf{1.86$\times$} \\
FP8 & 8.0 & 1 & 894.3 & 1.42$\times$ & 748.4 & 1.37$\times$ \\
\midrule
\multicolumn{7}{l}{\textit{Qwen3-32B}} \\
BF16 & 16.0 & 2 & 148.1 & 1.00$\times$ & 135.5 & 1.00$\times$ \\
\textbf{SLQ-DL} & \textbf{4.98} & \textbf{1} & \textbf{435.5} & \textbf{2.94$\times$} & \textbf{383.5} & \textbf{2.83$\times$} \\
FP8 & 8.0 & 1 & 282.8 & 1.91$\times$ & 121.6 & 0.90$\times$ \\
\midrule
\multicolumn{7}{l}{\textit{Llama-3.3-70B-Instruct}} \\
BF16 & 16.0 & 4 & 62.8 & 1.00$\times$ & 60.8 & 1.00$\times$ \\
\textbf{SLQ-DL} & \textbf{6.55} & \textbf{2} & \textbf{139.8} & \textbf{2.23$\times$} & \textbf{129.7} & \textbf{2.13$\times$} \\
\textbf{SLQ-TL} & \textbf{3.59} & \textbf{2} & \textbf{231.1} & \textbf{3.68$\times$} & \textbf{190.9} & \textbf{3.14$\times$} \\
FP8 & 8.0 & 4 & 93.8 & 1.50$\times$ & 88.2 & 1.45$\times$ \\
\bottomrule
\end{tabular}
\end{table}

\section{Conclusion \& Future Work}

We formalized statistically-lossless LLM compression, distinguishing task-lossless (TL) and distribution-lossless (DL) targets, and proposed EAR as a fidelity metric. The $\gamma^2$ variance law shows asymmetric quantization is a prerequisite for DL. SLQ achieves TL at 3.3--4.7 bits, DL at 5.0--6.6 bits, and $1.7$--$3.7\times$ speedup over BF16 across four models, and transfers unchanged to a mixture-of-experts model. Limitations are discussed in Appendix~\ref{app:limitations}; future work includes weight-and-activation quantization and speculative decoding.

\section*{Ethics Statement}

This paper presents research aimed at advancing the field of Machine Learning. We do not foresee any specific ethical concerns arising directly from this work that warrant particular emphasis beyond the broader societal implications common to the field.

\section*{Acknowledgments}
We acknowledge the use of Humming kernels developed by Jinzhen Lin and the Venus Team, Ant Group. 

\bibliography{references}
\bibliographystyle{colm2026_conference}

\appendix
\section*{Appendix}

\section{Algorithms and Metric Prediction}
\label{app:shapley-algorithm}

\subsection{Multi-Bitwidth Shapley Estimation}

Algorithm~\ref{alg:multi-shapley} details the multi-bitwidth Shapley estimation procedure described in Section~\ref{sec:shapley}.

\begin{algorithm}[H]
\caption{Multi-Bitwidth Shapley Estimation}
\label{alg:multi-shapley}
\begin{algorithmic}[1]
\REQUIRE Groups $\mathcal{G} = \{G_1, \ldots, G_M\}$, bitwidths $\mathcal{B}$, permutations $P$
\ENSURE Cost table $\phi_m^{(b)}$ for all groups $m$, bitwidths $b$
\FOR{each target bitwidth $b^* \in \mathcal{B} \setminus \{b_\text{max}\}$}
    \FOR{$p = 1, \ldots, P$}
        \STATE Sample random permutation $\pi$ of $\{1, \ldots, M\}$
        \STATE Set all groups to $b_\text{max}$; evaluate metrics
        \FOR{$j = 1, \ldots, M$}
            \STATE Switch group $G_{\pi_j}$ from $b_\text{max}$ to $b^*$
            \STATE Evaluate metrics (KL, EAR)
            \STATE Record marginal change $\Delta_{\pi_j}^{(b^*)}$
        \ENDFOR
    \ENDFOR
    \STATE $\phi_m^{(b^*)} \leftarrow \frac{1}{P} \sum_p \Delta_m^{(b^*)}$ for each group $m$
\ENDFOR
\end{algorithmic}
\end{algorithm}

\subsection{Metric Prediction from Sensitivity Data}
\label{app:metric_prediction}

During binary search, the predicted evaluation mode estimates KL divergence and EAR from sensitivity coefficients without running model inference. This enables rapid iteration over candidate bitwidth budgets.

\paragraph{Linear Method.} Under the independence assumption, total metric degradation is the sum of per-group contributions:
\begin{align}
    \widehat{\Delta D_{\text{KL}}}(\mathbf{b}) &= \sum_{m} e_m^{(b_m)} \cdot \alpha_m^{\text{KL}} \\
    \widehat{\Delta \text{EAR}}(\mathbf{b}) &= \sum_{m} e_m^{(b_m)} \cdot \alpha_m^{\text{EAR}}
\end{align}
where $e_m^{(b)}$ is the normalized reconstruction error for group $G_m$ at bitwidth $b$, computed as $e_m^{(b)} = \frac{1}{|G_m|} \sum_{\ell \in G_m} \|W_\ell - \hat{W}_\ell^{(b)}\|_F^2 / \|W_\ell\|_F^2$. The coefficients $\alpha_m^{\text{KL}}$ and $\alpha_m^{\text{EAR}}$ are estimated via noise injection as described in Section~\ref{sec:sensitivity-estimation}. The predicted metrics are then $\widehat{D_{\text{KL}}} = \widehat{\Delta D_{\text{KL}}}$ and $\widehat{\text{EAR}} = 1 - \widehat{\Delta \text{EAR}}$.

\paragraph{Shapley Method.} The Shapley values $\phi_m^{(b)}$ directly represent expected metric degradation from assigning group $m$ to bitwidth $b$ relative to the reference bitwidth $b_{\max}$:
\begin{align}
    \widehat{D_{\text{KL}}}(\mathbf{b}) &= \sum_{m} \phi_m^{\text{KL},(b_m)} \\
    \widehat{\text{EAR}}(\mathbf{b}) &= 1 - \sum_{m} \phi_m^{\text{EAR},(b_m)}
\end{align}
Since Shapley values are computed using actual quantized weights during permutation sampling, they capture both reconstruction error magnitude and its propagation through the network. For unified permutation sampling, which outputs transition costs rather than absolute costs, we apply cumulative summation from $b_{\max}$ downward before use in the solver.

\subsection{Task-Lossless Single-Point Calibration}
\label{app:tl-calibration}

Algorithm~\ref{alg:tl-search} details the single-point calibration procedure for task-lossless bitwidth search described in Section~\ref{sec:solver}.

\begin{algorithm}[h!]
\caption{Task-Lossless Single-Point Calibration Search}
\label{alg:tl-search}
\begin{algorithmic}[1]
\REQUIRE Sensitivity database $\phi_m^{(b)}$, ILP solver, target recovery $r^*$, calibration bitwidth $b_{\text{cal}}$
\ENSURE Minimum-bitwidth configuration $\mathbf{b}^*$ with recovery $\geq r^*$
\STATE \textbf{// Step 1: Calibrate}
\STATE $\mathbf{b}_{\text{cal}} \leftarrow \textsc{ILP-Solve}(\phi, \bar{b} = b_{\text{cal}})$ \COMMENT{Proxy-optimal config at $b_{\text{cal}}$}
\STATE $D_{\mathrm{KL}}^{\text{actual}} \leftarrow \textsc{MeasureKL}(\mathbf{b}_{\text{cal}})$ \COMMENT{One forward pass}
\STATE $D_{\mathrm{KL}}^{\text{predicted}} \leftarrow \sum_m \phi_m^{(b_{\text{cal},m})}$ \COMMENT{Shapley prediction}
\STATE $\rho \leftarrow D_{\mathrm{KL}}^{\text{actual}} / D_{\mathrm{KL}}^{\text{predicted}}$ \COMMENT{Calibration ratio}
\STATE $\text{recovery}_{\text{cal}} \leftarrow \textsc{Benchmark}(\mathbf{b}_{\text{cal}}) \,/\, \textsc{Benchmark}(\text{BF16})$
\STATE \textbf{// Step 2: Fit linear model}
\STATE $\alpha \leftarrow (1 - \text{recovery}_{\text{cal}}) \,/\, D_{\mathrm{KL}}^{\text{actual}}$ \COMMENT{Slope from single point}
\STATE $D_{\mathrm{KL}}^{\text{thresh}} \leftarrow (1 - r^*) \,/\, \alpha$ \COMMENT{KL threshold for target recovery}
\STATE \textbf{// Step 3: Binary search (no forward passes)}
\STATE $b_{\text{lo}} \leftarrow \min(\mathcal{B}),\; b_{\text{hi}} \leftarrow \max(\mathcal{B})$
\WHILE{$b_{\text{hi}} - b_{\text{lo}} > \epsilon$}
    \STATE $b_{\text{mid}} \leftarrow (b_{\text{lo}} + b_{\text{hi}}) / 2$
    \STATE $\mathbf{b} \leftarrow \textsc{ILP-Solve}(\phi, \bar{b} = b_{\text{mid}})$
    \STATE $\hat{D}_{\mathrm{KL}} \leftarrow \rho \cdot \sum_m \phi_m^{(b_m)}$ \COMMENT{Calibrated prediction}
    \IF{$\hat{D}_{\mathrm{KL}} \leq D_{\mathrm{KL}}^{\text{thresh}}$}
        \STATE $b_{\text{hi}} \leftarrow b_{\text{mid}}$;\; $\mathbf{b}^* \leftarrow \mathbf{b}$
    \ELSE
        \STATE $b_{\text{lo}} \leftarrow b_{\text{mid}}$
    \ENDIF
\ENDWHILE
\STATE \textbf{Return:} $\mathbf{b}^*$
\end{algorithmic}
\end{algorithm}

The calibration step (lines 2--6) requires one forward pass and two benchmark runs. The binary search (lines 10--19) uses only the sensitivity database and ILP solver, with no additional model evaluations.

\subsection{Evolutionary Search as an Alternative Solver}
\label{app:evo}

As an alternative to the ILP-based solver (Section~\ref{sec:solver}), bitwidth allocation can be formulated as a combinatorial search problem. We implement a variant of EvoPress~\citep{sieberling2025evopress} adapted to our constraint-based setting: whereas the original minimizes loss subject to a bitwidth budget, we minimize bitwidth subject to a quality constraint. Algorithm~\ref{alg:evo} summarizes the procedure.

\begin{algorithm}[H]
\caption{Constraint-Based Evolutionary Bitwidth Search}
\label{alg:evo}
\begin{algorithmic}[1]
\STATE \textbf{Input:} Groups $\mathcal{G}$, quality threshold $\tau$, offspring count $\lambda$
\STATE Initialize $\mathbf{b}^{(0)} \leftarrow$ uniform at $\max(\mathcal{B})$, \texttt{stall} $\leftarrow 0$
\WHILE{not converged}
    \STATE Generate $\lambda$ offspring via level-switch mutation
    \STATE Filter via multi-stage selection (increasing tokens per stage)
    \STATE $\mathbf{b}^* \leftarrow$ best offspring by fitness (include parent in final stage)
    \IF{$\mathbf{b}^*$ satisfies $D_{\text{KL}} < \tau$ and $\bar{b}(\mathbf{b}^*) < \bar{b}(\mathbf{b}^{(t)})$}
        \STATE $\mathbf{b}^{(t+1)} \leftarrow \mathbf{b}^*$, \texttt{stall} $\leftarrow 0$
    \ELSE
        \STATE \texttt{stall} $\leftarrow$ \texttt{stall} $+ 1$
    \ENDIF
    \IF{\texttt{stall} $\geq$ threshold}
        \STATE $\mathbf{b}^{(t+1)} \leftarrow \textsc{Curation}(\mathbf{b}^{(t)})$ \COMMENT{Bitwidth-neutral reallocation}
    \ENDIF
\ENDWHILE
\STATE \textbf{Return:} $\mathbf{b}^{(t)}$
\end{algorithmic}
\end{algorithm}

\paragraph{Fitness Function.} For an offspring $\mathbf{b}'$ with parent $\mathbf{b}$, let $\Delta b = \bar{b}(\mathbf{b}) - \bar{b}(\mathbf{b}')$ denote the bitwidth reduction. When $\Delta b > 0$, fitness is:
\begin{equation}
    f(\mathbf{b}') = \frac{D_{\text{KL}}(\mathbf{b}')}{\Delta b} \cdot \begin{cases} 1 & \text{if } D_{\text{KL}}(\mathbf{b}') \leq \tau \\ \gamma \cdot (D_{\text{KL}}(\mathbf{b}') - \tau) & \text{otherwise} \end{cases}
\end{equation}
where $\gamma \gg 1$ is a penalty coefficient. Lower fitness is better. Offspring that increase bitwidth ($\Delta b < 0$) receive heavily penalized fitness values.

\paragraph{Adaptive Curation.} The inverted formulation introduces a challenge: because the search only accepts bitwidth-reducing mutations, suboptimal early allocations cannot be corrected. When progress stalls, curation performs bitwidth-neutral swaps between equal-sized groups, keeping the average constant while reallocating capacity from insensitive to sensitive groups. This allows the search to escape local optima and continue compressing.

\paragraph{Limitations.} The evolutionary search requires orders of magnitude more forward passes than the ILP-based solver and does not scale to large models. We include it here as a reference formulation; all experimental results in the main paper use the sensitivity-based ILP solver.

\section{Derivation of the Gamma-Squared Variance Law}
\label{app:bussgang}

This appendix provides a self-contained derivation of Lemma~\ref{thm:gamma-squared}. We proceed in three steps: analyzing quantizer step sizes, deriving quantization noise variance via Bennett's high-rate approximation, and combining the two results.

\subsection{Definitions and Preliminaries}

Let $W$ be a scalar random variable representing a neural network weight, with probability density function supported on $[L, U]$ where $L < U$. We recall the geometric quantities from Definition~\ref{def:centering-inefficiency}: the dynamic range $R = U - L$, the maximum absolute magnitude $M = \max(|L|, |U|)$, and the centering inefficiency $\gamma = 2M/R$. For distributions that are not symmetric about zero, the symmetric interval $[-M, M]$ strictly contains $[L, U]$, so $\gamma > 1$.

\subsection{Step 1: Quantization Step Size Analysis}

Consider a uniform scalar quantizer with $n$ discrete levels (typically $n = 2^b$ for bitwidth $b$). The step size $\Delta$ is the distance between adjacent quantization levels.

\paragraph{Asymmetric quantizer.} The grid spans exactly $[L, U]$, dividing the range $R$ into $n - 1$ intervals:
\begin{equation}
    \Delta_{\mathrm{asym}} = \frac{R}{n - 1}.
\end{equation}

\paragraph{Symmetric quantizer.} The grid is constrained to be symmetric about zero, spanning $[-M, M]$:
\begin{equation}
    \Delta_{\mathrm{sym}} = \frac{2M}{n - 1}.
\end{equation}

\begin{lemma}[Step size scaling]
\label{lem:step-size}
The symmetric step size relates to the asymmetric step size by $\Delta_{\mathrm{sym}} = \gamma \, \Delta_{\mathrm{asym}}$.
\end{lemma}

\begin{proof}
Direct computation:
\begin{equation}
    \frac{\Delta_{\mathrm{sym}}}{\Delta_{\mathrm{asym}}} = \frac{2M/(n-1)}{R/(n-1)} = \frac{2M}{R} = \gamma.
\end{equation}
\end{proof}

\subsection{Step 2: Quantization Noise Variance}

We analyze the noise characteristics using classical results from quantization theory.

\paragraph{Bennett's approximation vs.\ Bussgang decomposition.} Two classical results characterize quantization noise. Bennett's approximation~\citep{bennett1948spectra} shows that, under the high-rate assumption (many quantization levels), the quantization error $e = Q(W) - W$ is approximately uniformly distributed within each bin, $e \sim \mathcal{U}[-\Delta/2, \, \Delta/2]$. The only requirement is that the probability density of $W$ is approximately constant within each bin, which is satisfied when the number of quantization levels is large relative to the variation in the density. The Bussgang decomposition~\citep{bussgang1952crosscorrelation} provides a complementary signal-processing perspective, decomposing the quantizer output as $Q(W) = \alpha W + d$ where $\alpha$ is a gain term and $d$ is noise uncorrelated with the input ($\mathbb{E}[Wd] = 0$). This uncorrelation guarantee, however, requires the input $W$ to be Gaussian. Since LLM weight distributions are generally non-Gaussian, we rely on Bennett's result, which applies under the weaker high-rate condition.

\begin{lemma}[Variance of quantization noise]
\label{lem:bussgang-variance}
Under the high-rate assumption, the variance of the quantization noise is $\sigma^2 = \Delta^2/12$.
\end{lemma}

\begin{proof}
The mean of $e \sim \mathcal{U}[-\Delta/2,\, \Delta/2]$ is zero. The variance is:
\begin{align}
    \sigma^2 
    &= \int_{-\Delta/2}^{\Delta/2} x^2 \cdot \frac{1}{\Delta} \, dx \\
    &= \frac{1}{\Delta} \left[ \frac{x^3}{3} \right]_{-\Delta/2}^{\Delta/2} \\
    &= \frac{1}{\Delta} \cdot \frac{2(\Delta/2)^3}{3} \\
    &= \frac{\Delta^2}{12}.
\end{align}
\end{proof}

\subsection{Step 3: Proof of the Gamma-Squared Law}

We combine Lemma~\ref{lem:step-size} and Lemma~\ref{lem:bussgang-variance} to establish the main result.

\begin{proof}[Proof of Lemma~\ref{thm:gamma-squared}]
Applying Lemma~\ref{lem:bussgang-variance} to each quantizer:
\begin{equation}
    \sigma_{\mathrm{asym}}^2 = \frac{\Delta_{\mathrm{asym}}^2}{12}, \qquad
    \sigma_{\mathrm{sym}}^2 = \frac{\Delta_{\mathrm{sym}}^2}{12}.
\end{equation}
Substituting $\Delta_{\mathrm{sym}} = \gamma \, \Delta_{\mathrm{asym}}$ from Lemma~\ref{lem:step-size}:
\begin{align}
    \sigma_{\mathrm{sym}}^2 
    &= \frac{(\gamma \, \Delta_{\mathrm{asym}})^2}{12} \\
    &= \gamma^2 \cdot \frac{\Delta_{\mathrm{asym}}^2}{12} \\
    &= \gamma^2 \, \sigma_{\mathrm{asym}}^2.
\end{align}
\end{proof}

The geometric penalty $\gamma$ imposed by symmetric quantization on asymmetric weight distributions therefore leads to a quadratic increase in quantization noise variance. For a single linear layer $\mathbf{y} = W\mathbf{x}$, the output perturbation due to quantization is $\boldsymbol{\delta} = E\mathbf{x}$ where $E = Q(W) - W$. The expected squared error per output unit is $\sigma^2 \|\mathbf{x}\|^2$. Since symmetric and asymmetric quantization are applied to the same weights with identical inputs, $\|\mathbf{x}\|^2$ cancels in the ratio, giving $\mathbb{E}[\|\boldsymbol{\delta}_{\mathrm{sym}}\|^2] / \mathbb{E}[\|\boldsymbol{\delta}_{\mathrm{asym}}\|^2] = \gamma^2$. This provides the analytical foundation for the empirical observations in Section~\ref{empirical:asym}, where symmetric quantization exhibits consistently higher flip rates and lower EAR.

\section{Additional Experiments and Results}
\label{app:additional-experiments}

We present additional experiments that support and extend the main results. Appendix~\ref{app:sensitivity-plots} shows the per-layer sensitivity landscape and how it translates into bitwidth assignments. Appendix~\ref{app:asym-empirical} quantifies the asymmetric-vs.-symmetric gap on a task-level benchmark. Appendix~\ref{app:speedups} measures per-layer kernel speedups across bitwidths. Appendix~\ref{app:kl_acc} characterizes how distributional shift (KL) relates to zero-shot accuracy, connecting the distribution-lossless (DL) and task-lossless (TL) operating points. Appendix~\ref{app:actual_bitwidths} reports exact effective bits per parameter accounting for all overheads. Appendix~\ref{app:token-divergence} inspects token-level divergence under DL compression. Appendix~\ref{app:detailed-results} gives the full sweep of quantization configurations, formats, and allocation strategies.

\subsection{Sensitivity Estimates and Bitwidth Assignments}
\label{app:sensitivity-plots}

Figures~\ref{fig:sensitivity-continuous} and~\ref{fig:sensitivity-discrete} illustrate the two stages of the SLQ pipeline for Qwen3-8B: sensitivity estimation and bitwidth allocation. Both methods agree on the broad sensitivity structure (attention K/V most sensitive, MLP least), but differ in how they capture it, which affects the final allocation.

Figure~\ref{fig:sensitivity-continuous} shows the continuous per-layer EAR sensitivity coefficients from linear estimation (Section~\ref{sec:sensitivity}). The coefficients reflect each layer's independent contribution to output degradation: K and V projections are consistently the most sensitive, while MLP gate/up projections and the output projection are less so. Early and late layers tend to be more sensitive than middle layers.

Figure~\ref{fig:sensitivity-discrete} shows the discrete bitwidth assignment produced by the ILP solver using Shapley-based sensitivity estimates, targeting distribution-lossless compression (5.70 bits average, EAR~$= 99.0\%$). Because Shapley values account for cross-layer interactions during permutation sampling, the resulting allocation can differ from what linear estimates alone would produce. The solver assigns 8 bits to the most sensitive K/V projections, 6--7 bits to Q and output projections, and 4--5 bits to the more robust MLP layers.

\begin{figure}[H]
  \centering
  \includegraphics[width=\linewidth]{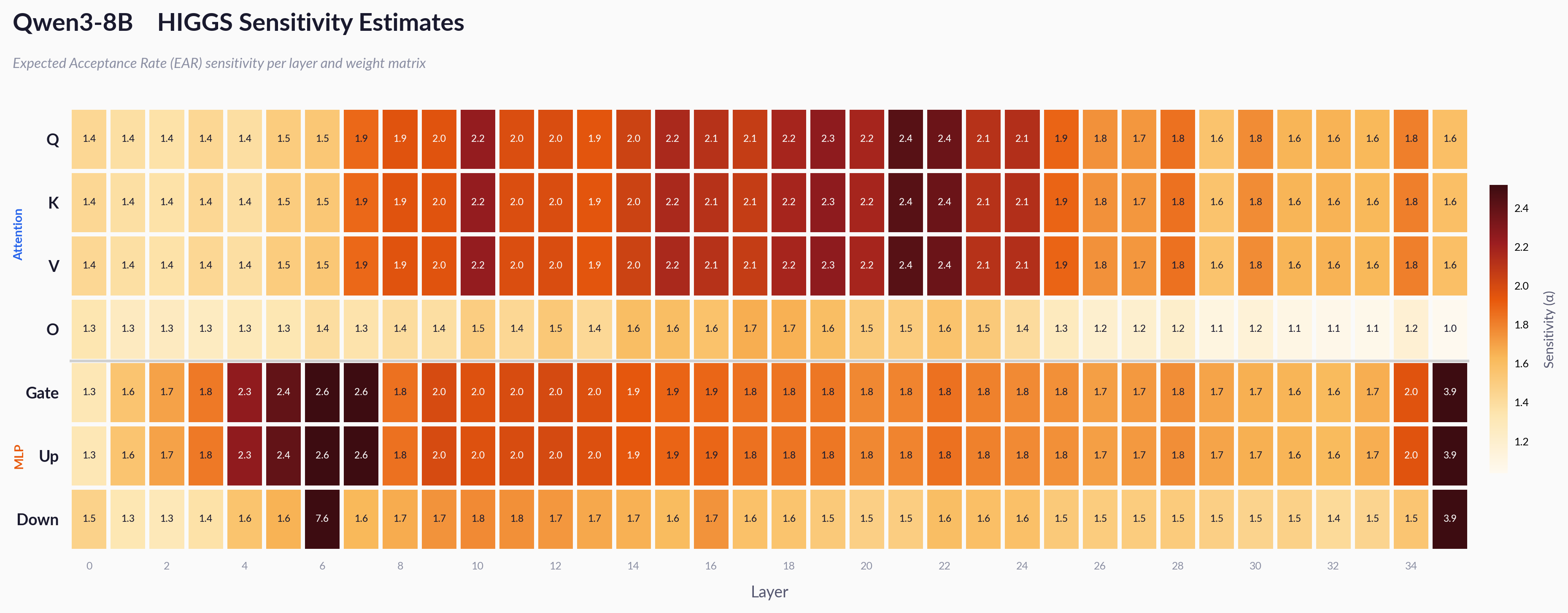}
  \caption{Per-layer EAR sensitivity coefficients from linear estimation for Qwen3-8B. Continuous values represent each layer's independent contribution to output degradation.}
  \label{fig:sensitivity-continuous}
\end{figure}

\begin{figure}[H]
  \centering
  \includegraphics[width=\linewidth]{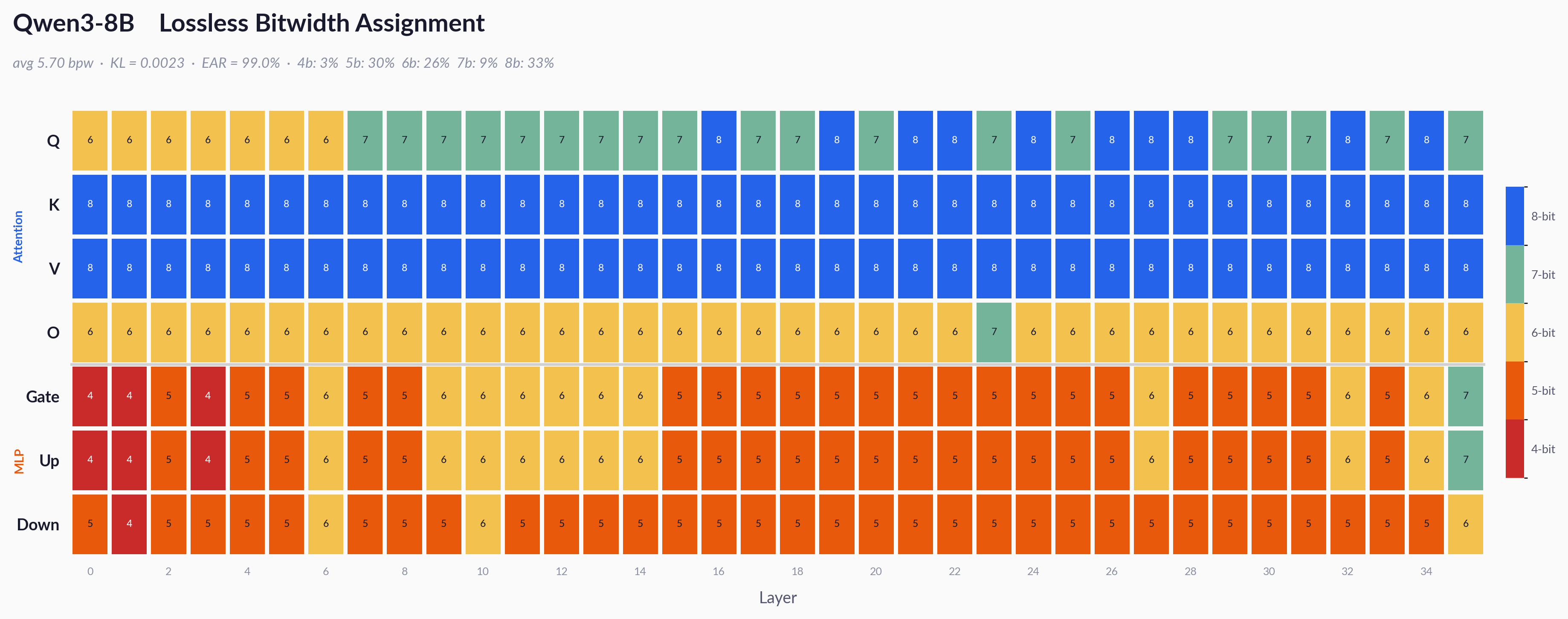}
  \caption{Bitwidth assignment from Shapley-based allocation for Qwen3-8B DL configuration (5.70 avg bits, EAR~$= 99.0\%$, KL~$= 0.0023$). The solver assigns higher bitwidths to sensitive attention layers and lower bitwidths to robust MLP layers.}
  \label{fig:sensitivity-discrete}
\end{figure}

\subsection{Asymmetric vs.\ Symmetric: Additional Results}
\label{app:asym-empirical}

Table~\ref{tab:mmlu-results} provides the full MMLU-Pro breakdown referenced in Section~\ref{empirical:asym}. The task reduces to a single-token decision (A/B/C/D), so every decision flip directly changes the answer. The model is instructed to output only the final answer token; overall scores are therefore lower than with chain-of-thought reasoning.

\begin{table}[H]
\centering
\caption{MMLU-Pro zero-shot (12,032 samples), uniform 4-bit, Llama-3.1-8B-Instruct.}
\label{tab:mmlu-results}
\small
\begin{tabular}{@{}lccc@{}}
\toprule
Metric & FP16 & Asymmetric & Symmetric \\
\midrule
Accuracy & 18.09\% & \textbf{15.99\%}  & 14.92\% \\
Recovery rate & 100\% & \textbf{88.4\%}  & 82.5\% \\
\midrule
Mean margin on flip & 0 & \textbf{0.162} & 0.175  \\
\bottomrule
\end{tabular}
\end{table}

The asymmetric--symmetric gap is most pronounced for distribution-lossless (DL) compression, where EAR is the target metric. For task-lossless (TL) compression, symmetric quantization may still suffice since benchmark scores absorb additional token flips without measurably affecting aggregate accuracy (Section~\ref{empirical:asym}). Full EAR trajectories across bitwidths and additional results on margin at disagreement appear in Appendix~\ref{app:detailed-results}.

\subsection{Comparison to Stretched Elastic Quantization}
\label{app:seq}

Stretched Elastic Quantization (SEQ)~\citep{liu2025paretoq} addresses a \emph{level-allocation} inefficiency: it spaces levels evenly across the span, which omits a zero level at even level counts ($2$-bit: $\{\pm 0.5, \pm 1.5\}$). The $\gamma^2$ law concerns a different, \emph{centering} inefficiency, and the two are independent. SEQ learns only the scale $\alpha$, which rescales the grid but never shifts it off zero, so a SEQ grid remains centered and cannot fit itself to an off-center $[L, U]$; only a zero-point realizes the finer step of Lemma~\ref{thm:gamma-squared}. SEQ therefore gives no relief on the centering axis, and the choice between the two is not a trade-off. Consistently, ParetoQ applies SEQ only at ternary and 2-bit precision, where scarce levels make their allocation the binding constraint, and reverts to the zero-inclusive LSQ grid at 3 and 4 bit.

\subsection{Speedup with Humming Kernels}
\label{app:speedups}

\begin{figure}[H]
\centering
\begin{tikzpicture}
\begin{axis}[
    width=0.85\textwidth,
    height=5.5cm,
    ybar=0.1cm,
    bar width=0.35cm,
    ylabel={Speedup ($\times$)},
    xlabel={Weight Bitwidth},
    symbolic x coords={4-bit, 5-bit, 6-bit, 7-bit, 8-bit},
    xtick=data,
    ymin=0,
    ymax=4.5,
    ymajorgrids=true,
    grid style={dashed, gray!40},
    legend style={
        at={(0.5,-0.18)},
        anchor=north,
        legend columns=2,
        /tikz/every even column/.append style={column sep=0.5cm}
    },
    enlarge x limits=0.15,
    nodes near coords,
    every node near coord/.append style={font=\footnotesize, rotate=45, anchor=west},
    nodes near coords align={vertical},
    point meta=explicit symbolic,
]

\addplot[fill=color1, draw=color1!80!black] coordinates {
    (4-bit, 3.31) [3.31$\times$]
    (5-bit, 2.82) [2.82$\times$]
    (6-bit, 2.49) [2.49$\times$]
    (7-bit, 2.20) [2.20$\times$]
    (8-bit, 2.01) [2.01$\times$]
};

\addplot[fill=color2, draw=color2!80!black] coordinates {
    (4-bit, 2.89) [2.89$\times$]
    (5-bit, 2.53) [2.53$\times$]
    (6-bit, 2.24) [2.24$\times$]
    (7-bit, 2.00) [2.00$\times$]
    (8-bit, 1.85) [1.85$\times$]
};

\legend{Gate/Up Proj (57344$\times$8192)  Down Proj (8192$\times$28672)}
\end{axis}
\end{tikzpicture}
\caption{Speedup relative to FP16 baseline for different quantization bitwidths, when executing layers of Llama-3-70B. }
\label{fig:speedup}
\end{figure}

\subsection{KL Divergence and Zero-Shot Performance}
\label{app:kl_acc}

To establish the relationship between output fidelity and downstream task performance, we conduct controlled experiments that systematically vary the degree of distributional shift from the original model. We inject uniform noise into model weights at increasing levels, where noise magnitude is scaled by each weight matrix's standard deviation to approximate quantization error characteristics. For each noise level, we compute KL divergence between the perturbed and original model's output distributions, then evaluate zero-shot accuracy across five task categories: Reasoning (ARC, HellaSwag, WinoGrande, PIQA), Knowledge (MMLU, TriviaQA), Math (GSM8K, MATH), Code (HumanEval, MBPP), and Retrieval (NeedleBench).

Figure~\ref{fig:kl_vs_acc} presents degradation curves across six models spanning different scales and architectures. The results reveal a consistent pattern: accuracy remains stable while KL divergence remains small, then degrades as divergence increases. The specific onset of degradation varies with model scale and task category, with Retrieval tasks demonstrating the greatest robustness and Math and Code proving more sensitive to distributional shift. Crucially, all models maintain baseline accuracy when KL divergence stays below 0.01. This establishes the link between distribution-lossless (DL) and task-lossless (TL) compression: a model that is DL (high distributional fidelity, as measured by EAR or KL) is also TL (preserves zero-shot task accuracy). The converse does not hold: a model can be TL at a lower bitwidth without meeting the stricter DL criterion, since benchmark scores are less sensitive than distributional metrics (Section~\ref{sec:solver}).

Figure~\ref{fig:kl_acc_linear} zooms into the near-lossless regime on a linear scale, complementing the log-scale overview of Figure~\ref{fig:kl_acc_all}. The approximately linear decline of accuracy with KL divergence in this range is the property exploited by single-point calibration (Section~\ref{sec:solver}): it is what allows a single anchor measurement to determine the slope $\alpha$, and hence the KL threshold corresponding to any target recovery.

\begin{figure}[H]
    \centering
    \includegraphics[width=0.62\textwidth]{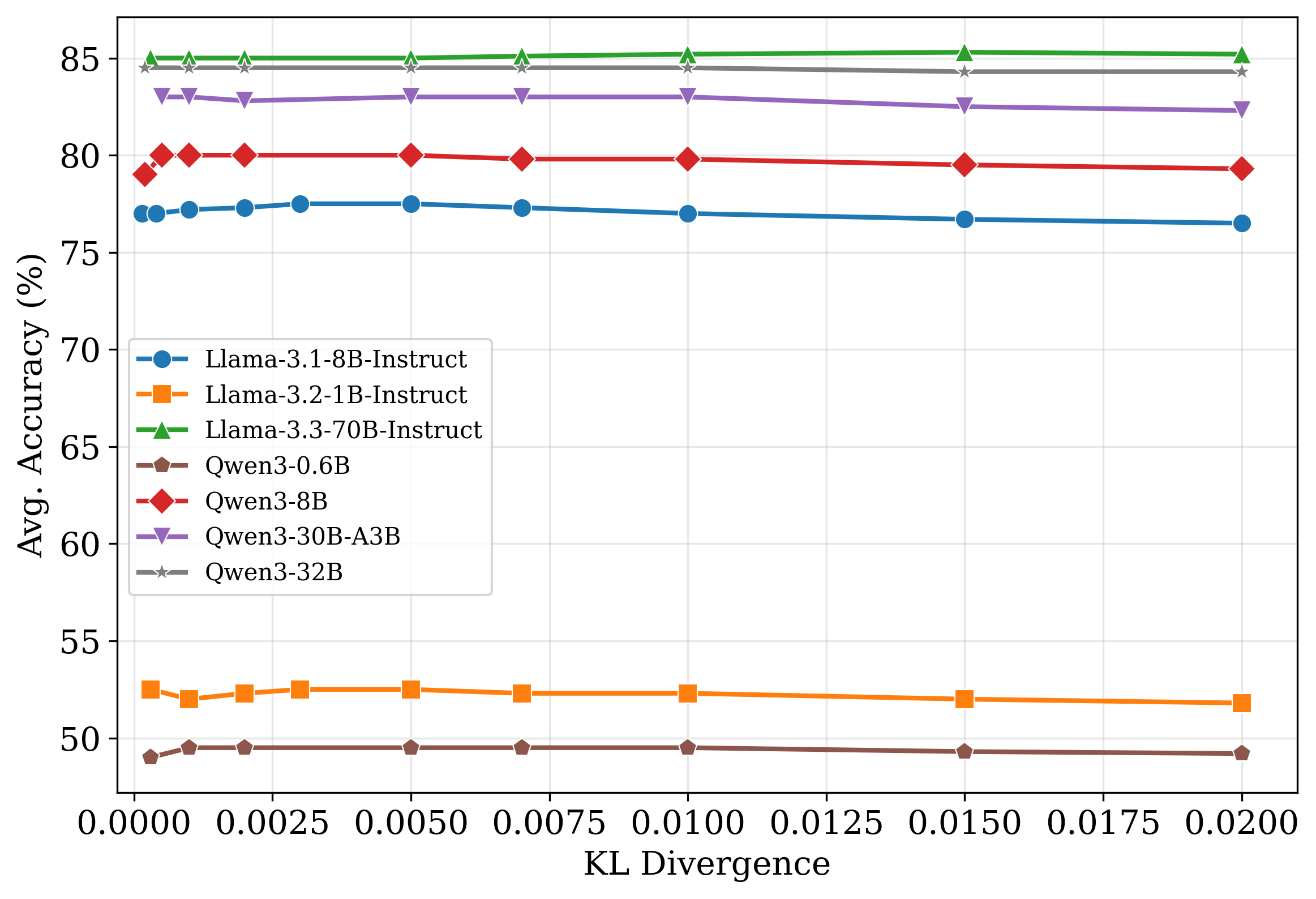}
    \caption{Near-lossless regime on a linear scale. Accuracy declines approximately linearly with KL divergence, which single-point calibration exploits to derive a KL threshold from one benchmark measurement.}
    \label{fig:kl_acc_linear}
\end{figure}

\begin{table}[h]
\centering
\caption{KL divergence on Llama-3.1-8B-Instruct. Methods above the rule operate at uniform 4-bit precision; SLQ uses non-uniform scalar integer allocation across $\{4,5,6,7,8\}$ bits (Section~\ref{sec:methods}). GPTQ, QTIP and SLQ are measured by us on the released checkpoints under an identical evaluation setup. $\dagger$: value as reported by~\citet{tseng2025yaqa}, which we did not reproduce on this model; our own YAQA measurements on Llama-3.3-70B are in Table~\ref{tab:sota-comparison}.}
\label{tab:kl-comparison}
\small
\begin{tabular}{@{}lccc@{}}
\toprule
Method & Bits & KL $\downarrow$ & EAR $\uparrow$ \\
\midrule
GPTQ (asym, g128) & 4.16 & 0.028 & 0.938 \\
QTIP & ${\sim}$4 & 0.017 & 0.952 \\
YAQA$^{\dagger}$~ & ${\sim}$4 & 0.028 & -- \\
\midrule
{SLQ} & {5.59} & \textbf{0.004} & \textbf{0.994} \\
\bottomrule
\end{tabular}
\end{table}

\begin{figure}[t]
    \centering
    \begin{subfigure}[t]{0.32\textwidth}
        \centering
        \includegraphics[width=\textwidth]{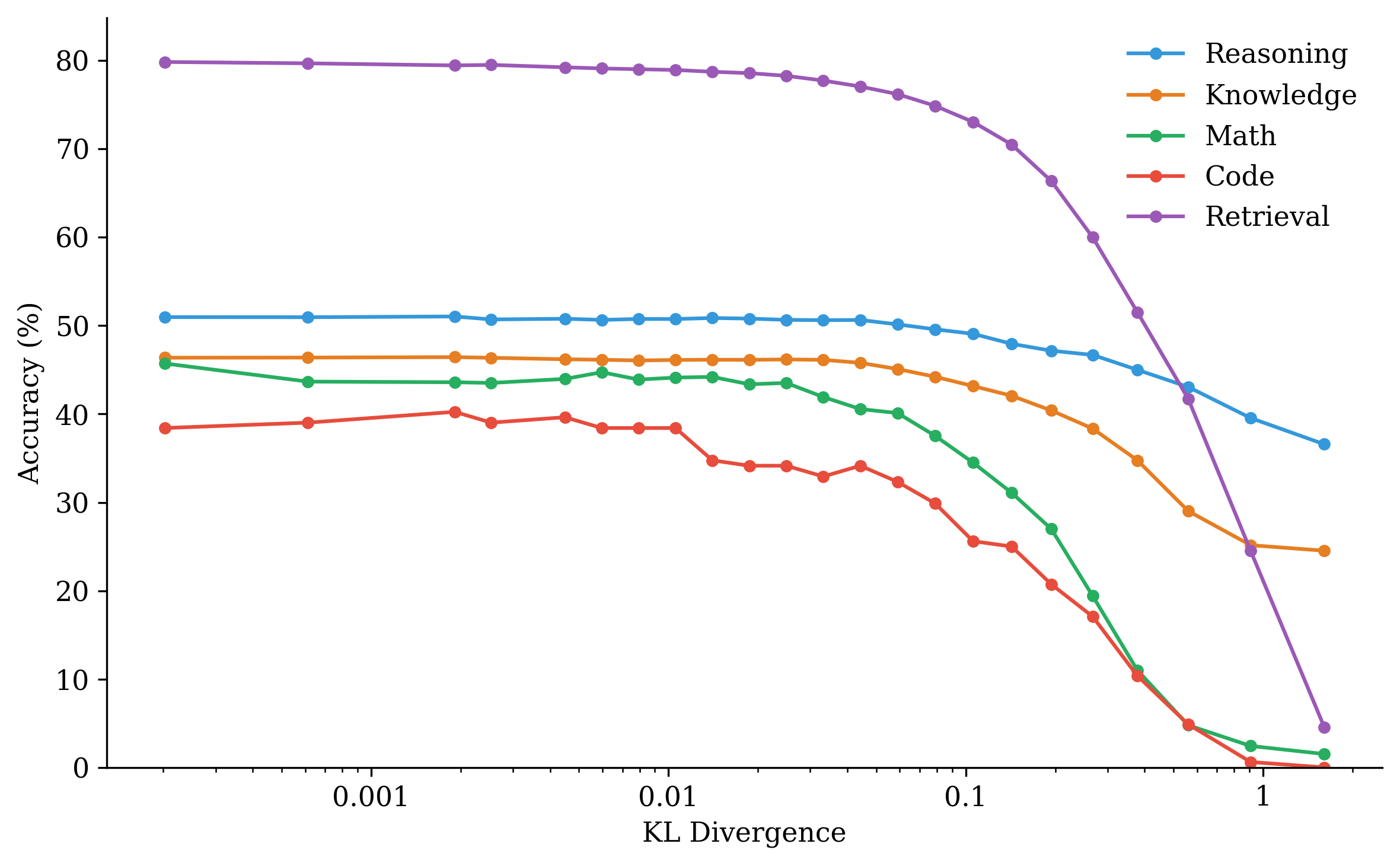}
        \caption{Llama-3.2-1B-Instruct}
    \end{subfigure}
    \hfill
    \begin{subfigure}[t]{0.32\textwidth}
        \centering
        \includegraphics[width=\textwidth]{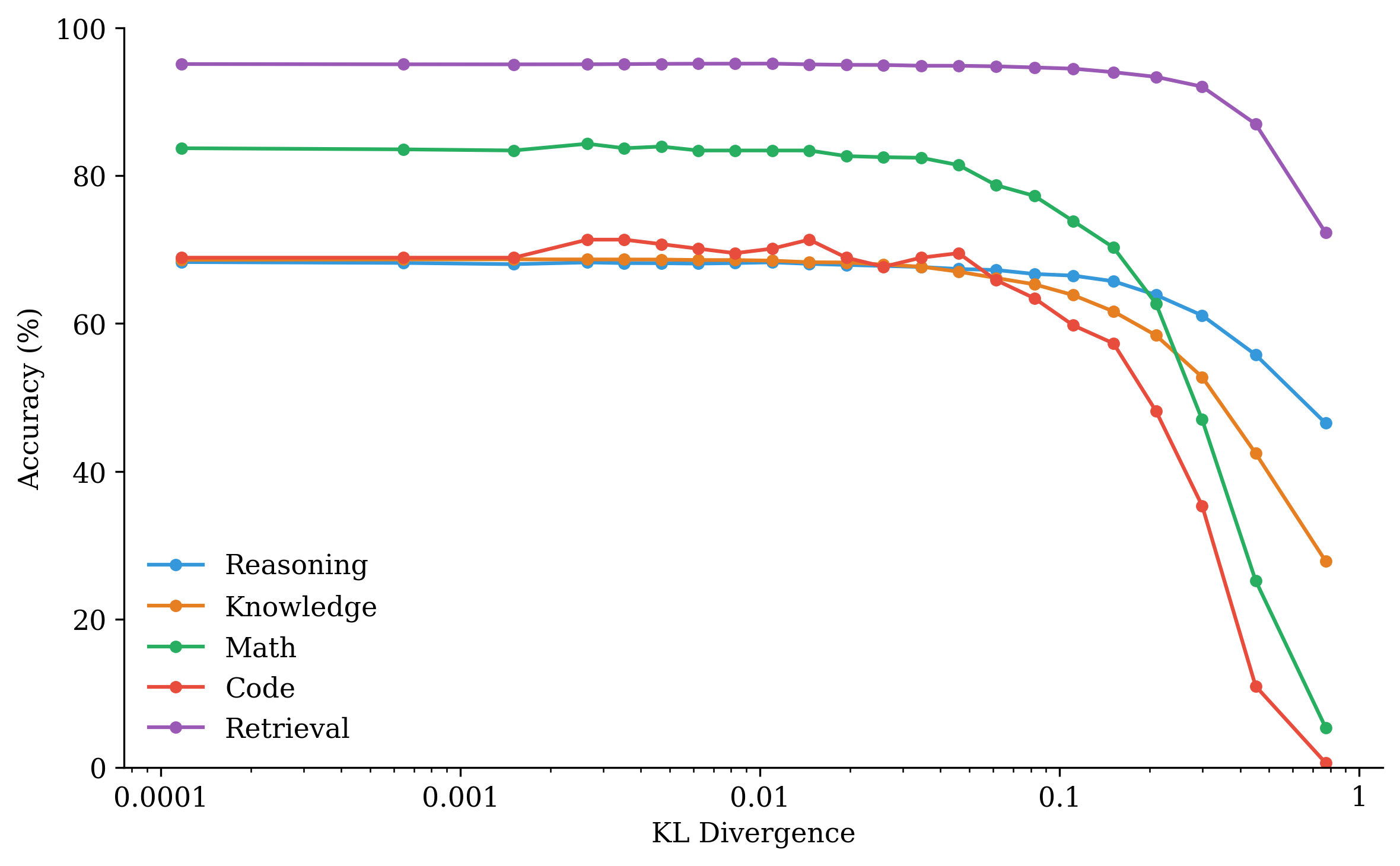}
        \caption{Llama-3.1-8B-Instruct}
    \end{subfigure}
    \hfill
    \begin{subfigure}[t]{0.32\textwidth}
        \centering
        \includegraphics[width=\textwidth]{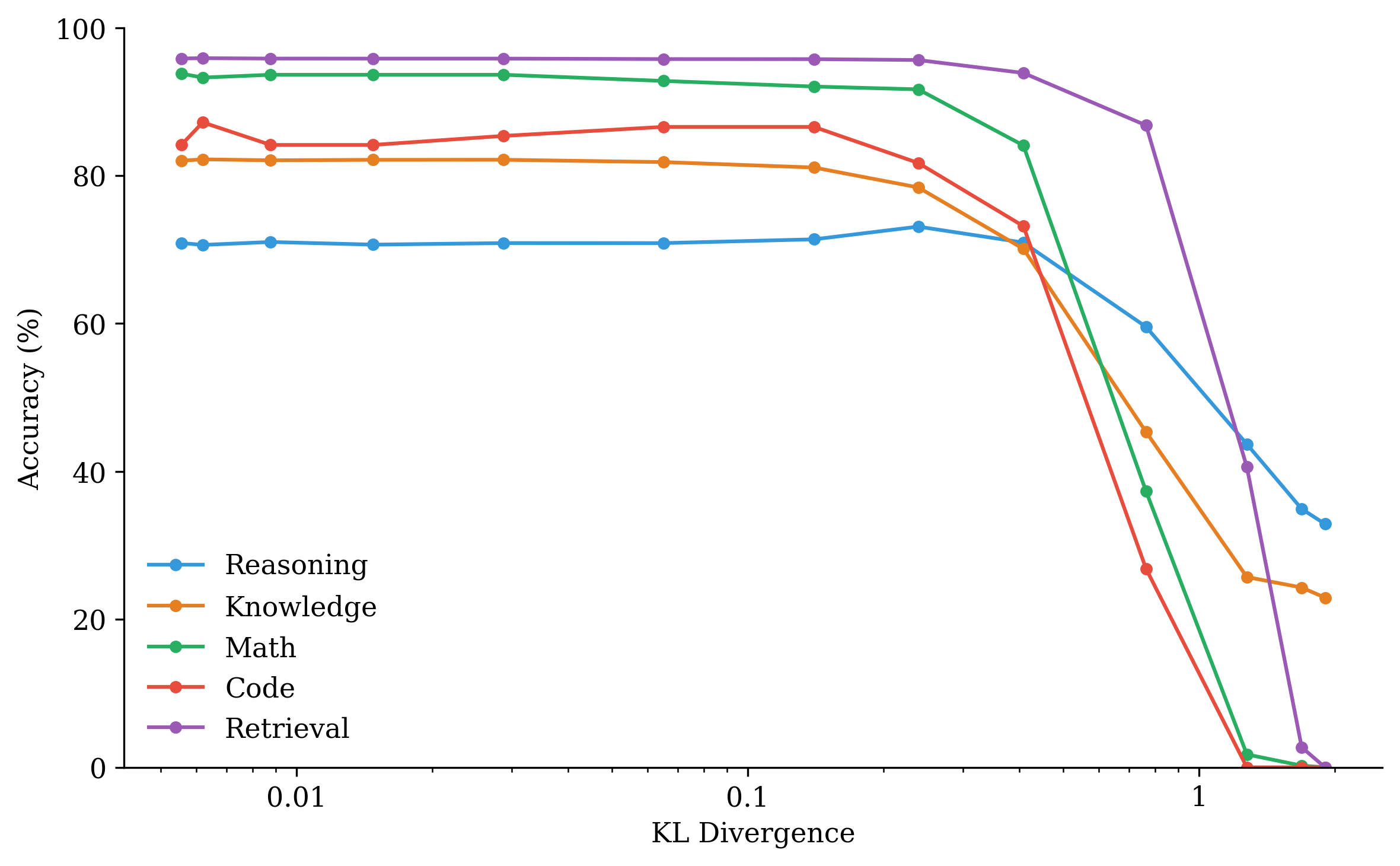}
        \caption{Llama-3.3-70B-Instruct}
    \end{subfigure}
    
    \vspace{0.5em}
    
    \begin{subfigure}[t]{0.32\textwidth}
        \centering
        \includegraphics[width=\textwidth]{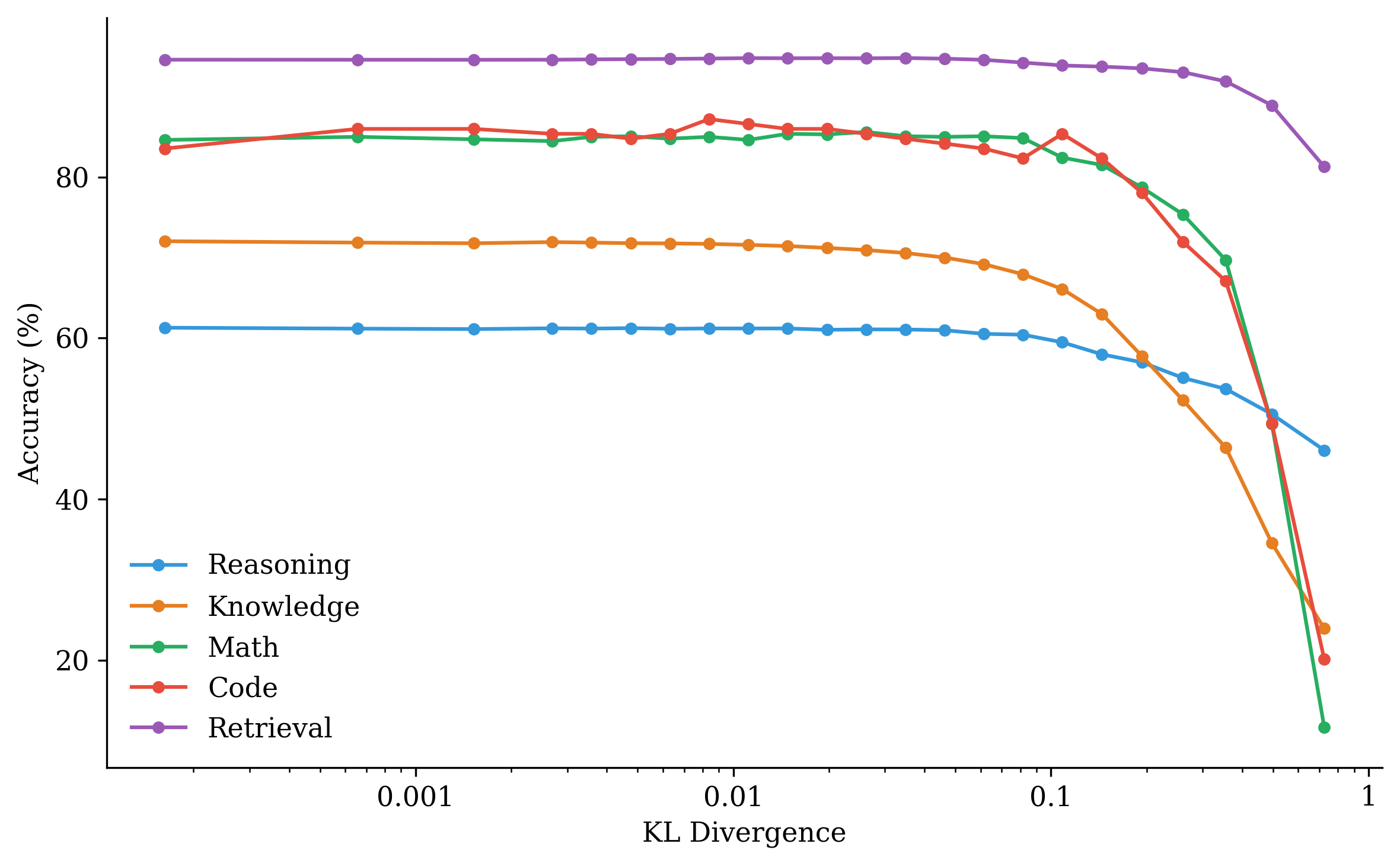}
        \caption{Qwen3-8B}
    \end{subfigure}
    \hfill
    \begin{subfigure}[t]{0.32\textwidth}
        \centering
        \includegraphics[width=\textwidth]{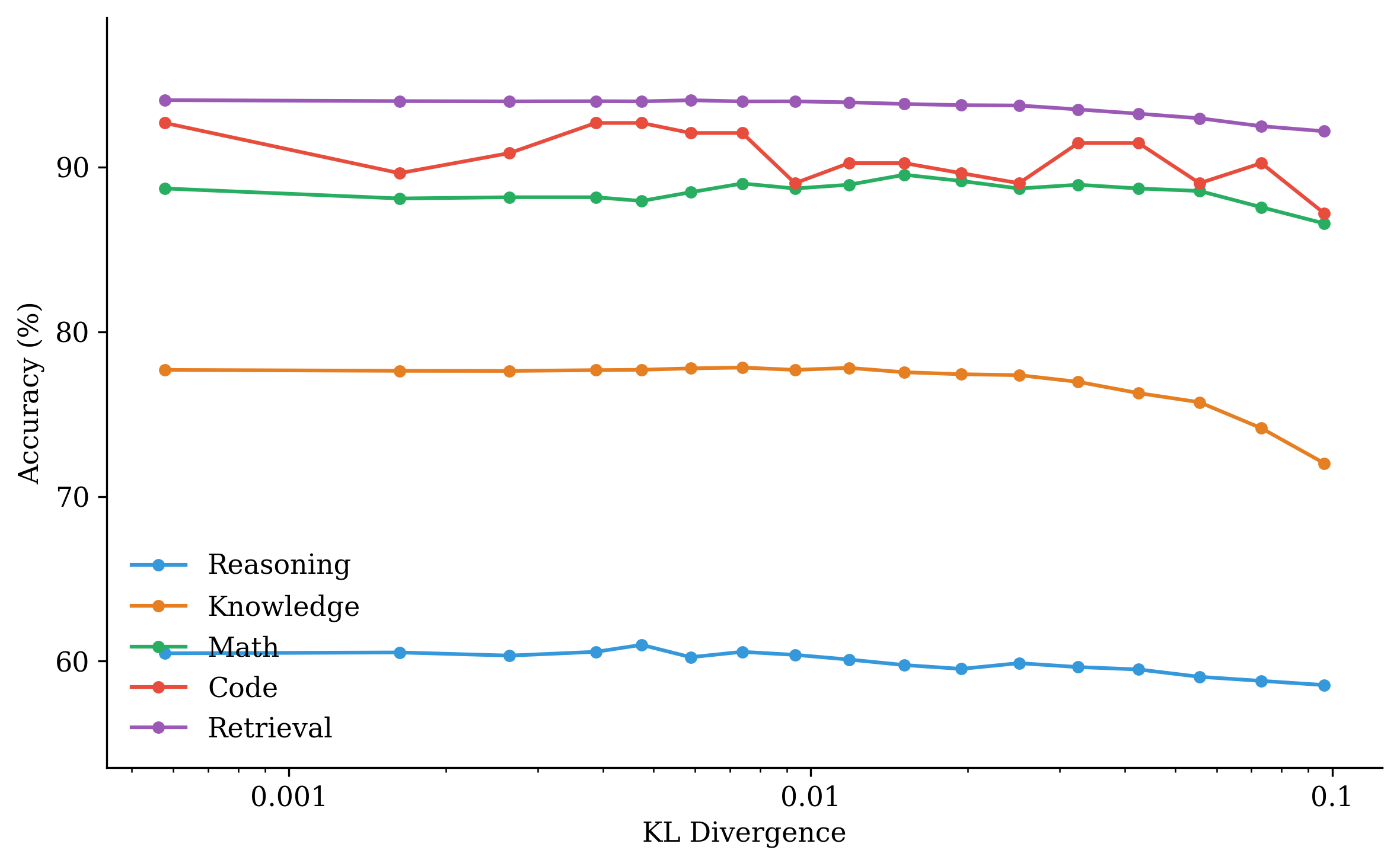}
        \caption{Qwen3-30B-A3B}
    \end{subfigure}
    \hfill
    \begin{subfigure}[t]{0.32\textwidth}
        \centering
        \includegraphics[width=\textwidth]{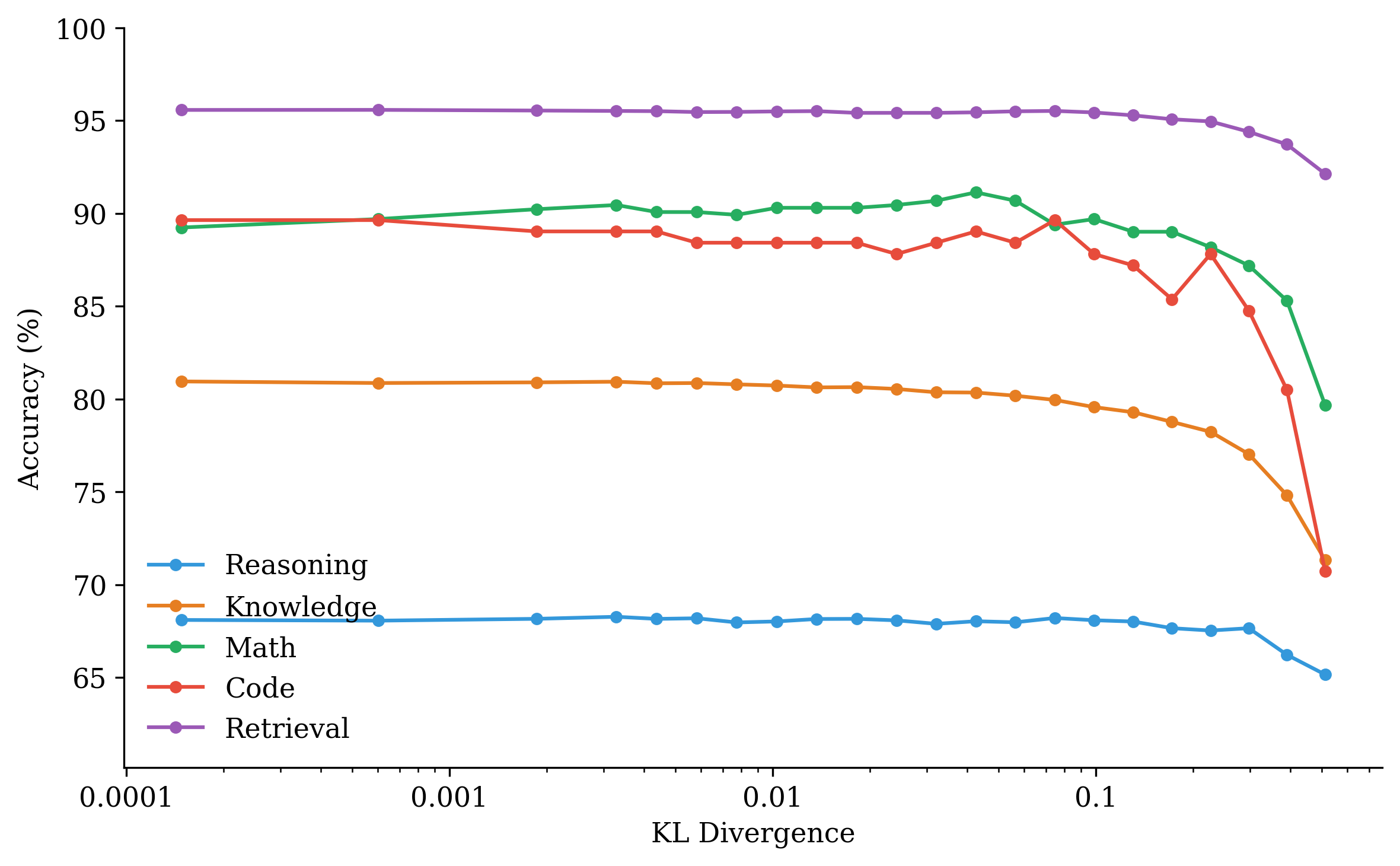}
        \caption{Qwen3-32B}
    \end{subfigure}
    
    \caption{Zero-shot accuracy versus KL divergence across task categories for different model architectures and sizes. Accuracy remains stable until KL divergence exceeds approximately 0.01: distribution-lossless (DL) compression implies task-lossless (TL) compression, but TL can be achieved at lower bitwidths since benchmarks are less sensitive than distributional metrics. Larger models tolerate greater distributional shift before degradation. The x-axis uses logarithmic scale.}

    \label{fig:kl_vs_acc}
\end{figure}
\section{Comparison to State-of-the-Art Quantizers}
\label{app:sota}

This appendix details the evaluation behind Table~\ref{tab:sota-comparison}. We compare against two strong per-layer quantizers that target the near-lossless frontier: QTIP~\citep{tseng2024qtip}, a trellis-coded quantizer, and YAQA~\citep{tseng2025yaqa}, which uses full-model-KL-aware Hessian sketching. Both are orthogonal to SLQ in the sense that they replace the per-layer quantizer while leaving the allocation question open; our framework would apply on top of either, at the cost of substantially heavier system support.

\noindent\textbf{Setup.} We evaluate the publicly released checkpoints for Llama-3.1-8B-Instruct and Llama-3.3-70B-Instruct rather than quoting published numbers, so that tasks, prompts, sampling parameters, seeds and the KL/EAR calibration set match those used for our own models exactly (Section~\ref{sec:experiments}). Task mean recovery is the average over the seven benchmarks of Table~\ref{tab:main-accuracy}, relative to the BF16 baseline, under non-greedy decoding averaged over 3 seeds. KL and EAR are computed on the same calibration positions as for SLQ, restricted to the top-$K{=}10$ tokens.

\noindent\textbf{Findings.} Neither method reaches the distribution-lossless regime at any tested uniform bitwidth. On Llama-3.1-8B, QTIP at 4-bit attains 97.94\% mean recovery, which misses the task-lossless threshold of 99\% outright, with EAR $0.9519$; at 3-bit the model is visibly lossy (96.00\%, EAR $0.9165$). On Llama-3.3-70B, both QTIP (99.9\%) and YAQA (99.5\%) clear the task-lossless bar at 4-bit, but their EAR values of $0.9544$ and $0.9622$ remain far below the DL threshold of $0.99$. SLQ reaches task-losslessness on the same model at 3.59\,bpp, i.e.\ it clears the same threshold at a lower bitwidth, and reaches DL at 6.55\,bpp.

These quantizers are not weak: at a fixed 4-bit budget they are considerably stronger than GPTQ, which is why we adopt them as the reference point. The uniform bit budget is instead selected without reference to any fidelity target, and therefore overshoots it on one model while undershooting it on another. Choosing the bitwidth from the fidelity target instead, which is what SLQ does, is what makes the near-lossless operating point reachable at all.

\subsection{Comparison to Non-Uniform Allocators}
\label{app:allocators}

Figure~\ref{fig:allocator-comparison} accompanies the discussion in Section~\ref{sec:accuracy}. We compare SLQ against IMPQ~\citep{zhao2025impq} on Qwen3-8B, sweeping both methods across bitrates and recording average task recovery and EAR at each operating point. IMPQ is run in its published configuration, which restricts allocation to the binary choice $\{4, 8\}$.

The two panels separate the two effects. On average task recovery the methods are close above ${\sim}4$\,bpp, both saturating near the BF16 baseline; benchmark accuracy simply cannot resolve configurations at this level of fidelity, which is the motivation for adopting EAR as the primary search objective. On EAR the curves separate, and IMPQ's binary allocation does not reach the DL threshold anywhere in the tested range, consistent with the bitwidth-range ablation in Section~\ref{sec:accuracy}. Beyond the gap itself, neither IMPQ nor the other non-uniform allocators provide a mechanism to identify where losslessness begins: locating the green operating points with a budget-driven method requires sweeping the entire curve and running benchmark evaluations at every budget, whereas SLQ obtains them directly from a fidelity constraint evaluated on calibration data.

\begin{figure}[H]
  \centering
  \begin{subfigure}[t]{0.48\linewidth}
    \centering
    \includegraphics[width=\linewidth]{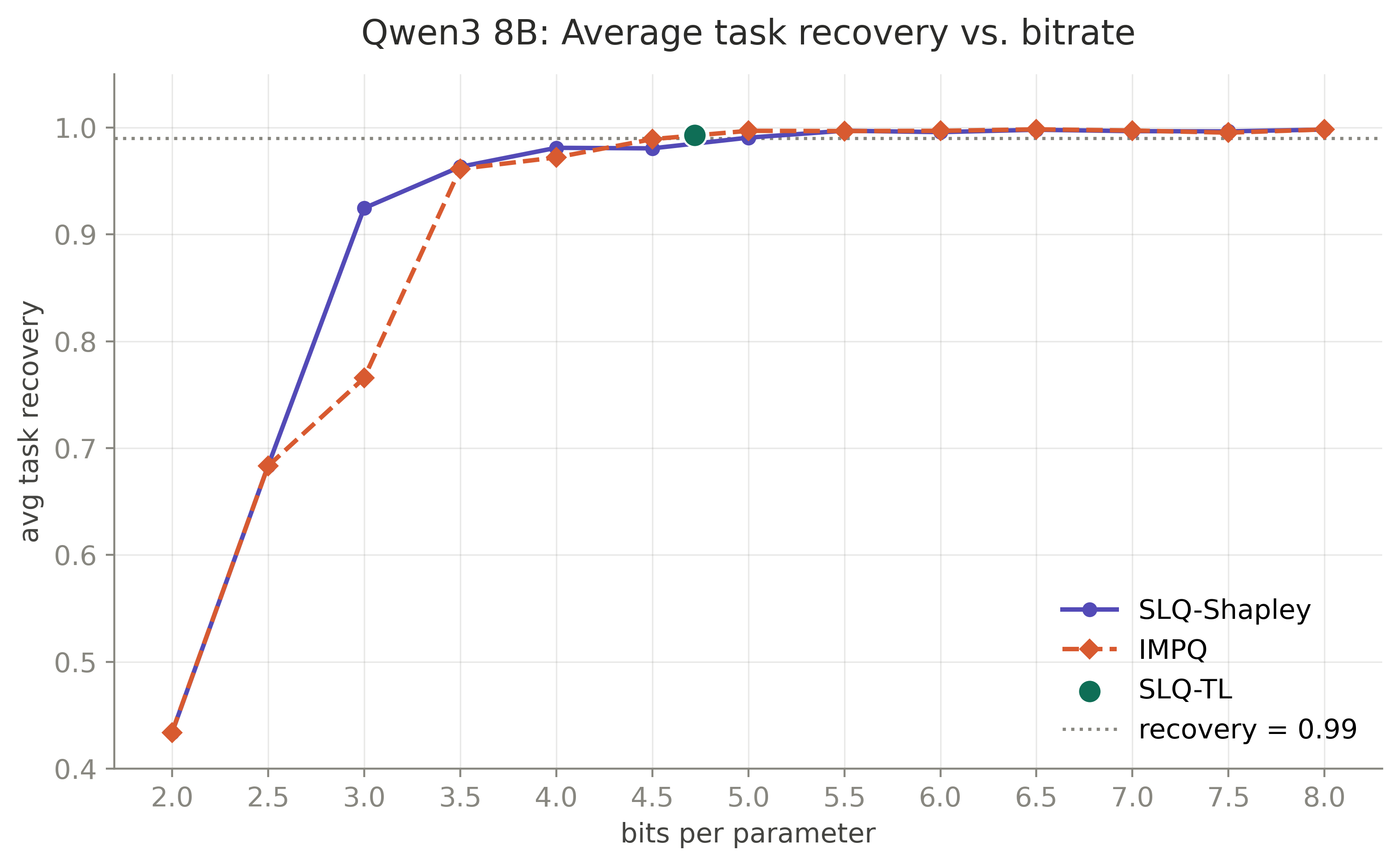}
    \caption{Average task recovery vs.\ bitrate.}
    \label{fig:impq-recovery}
  \end{subfigure}
  \hfill
  \begin{subfigure}[t]{0.48\linewidth}
    \centering
    \includegraphics[width=\linewidth]{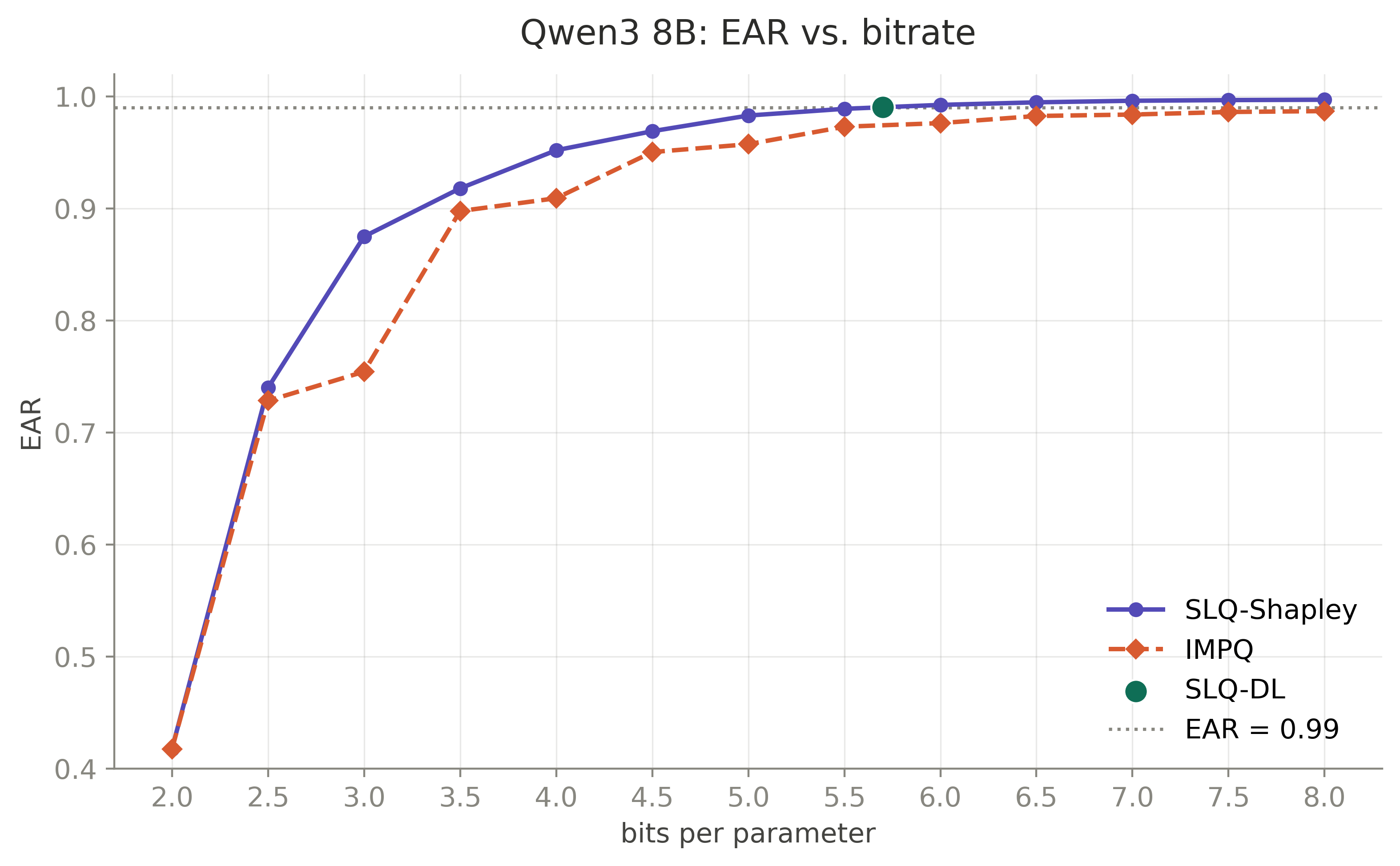}
    \caption{EAR vs.\ bitrate.}
    \label{fig:impq-ear}
  \end{subfigure}
  \caption{SLQ versus the IMPQ non-uniform allocator on Qwen3-8B. Green dots mark the SLQ-TL (a) and SLQ-DL (b) operating points selected by our fidelity-targeted search; dotted lines mark the corresponding thresholds. Task recovery saturates for both methods, while EAR separates them.}
  \label{fig:allocator-comparison}
\end{figure}

\section{Scope and Limitations}
\label{app:limitations}

This appendix expands on the three limitations summarized in the conclusion.

\noindent\textbf{Calibration cost.} Multi-bitwidth Shapley estimation runs $O(P \cdot M \cdot |\mathcal{B}|)$ forward passes, against $O(M \cdot |\mathcal{B}|)$ for the linear estimator of~\citet{malinovskii2025higgs}. The ratio of $10$--$100\times$ is independent of model size, but it multiplies a per-pass cost that is not. On an 8B model this amounts to roughly 2.5\,h per bitwidth on a single A6000, or ${\sim}15$\,h across the six non-reference bitwidths; the games are independent, so the wall-clock cost divides by the number of available GPUs. At $100$B$+$ scale the absolute cost becomes prohibitive and linear estimation is the appropriate choice.

The trade-off is compute against bitwidth rather than compute against fidelity. Both estimators drive the search to the same fidelity target, so both produce statistically-lossless models; Shapley's advantage is that the resulting model is slightly smaller, an advantage concentrated in the 5.5--6.5\,bpp band where allocation decisions are tightest (Appendix~\ref{app:detailed-results}). Linear estimation therefore costs bitwidth efficiency, but not losslessness.

\noindent\textbf{Weight-and-activation quantization.} SLQ is agnostic to activation quantization. Sensitivity estimation already incorporates activation noise, and neither the ILP formulation nor the EAR objective changes when activations are quantized, so W$+$A requires no algorithmic modification. We nonetheless report W$+$A only at small scale, because hardware-supported activation formats remain far more limited than weight formats: the available options force a coarser grid than the full $\{2,\dots,8\}$ range on which near-lossless compression depends (Appendix~\ref{app:detailed-results}). This is a property of the current kernel ecosystem rather than of the method, and the restriction lifts as activation format support matures.

\noindent\textbf{Choice of per-layer quantizer.} We use GPTQ throughout because it is well-tuned and the de facto standard for vLLM deployment. The $\gamma^2$ law and the TL/DL framework apply to any layer-wise quantizer, including vector- and trellis-coded methods. In the near-lossless regime, non-uniform allocation and grid asymmetry are the dominant factors; a higher-fidelity per-layer quantizer helps at the margin, as Table~\ref{tab:sota-comparison} indicates, but does not by itself move a uniform configuration across the losslessness boundary.

\noindent\textbf{Attention architectures.} All evaluated models (the Qwen3 family and Llama-3.3-70B) use grouped-query attention. Attention projections account for a small fraction of total parameters, so differences in per-group sensitivity structure between GQA and multi-head attention do not materially affect the resulting average bitwidth or throughput.

\section{Mixture-of-Experts Models}
\label{app:moe}

To test whether the SLQ pipeline transfers to sparsely-activated architectures, we evaluate Qwen3.5-35B-A3B (35B total parameters, 3B active), the successor to the MoE model appearing in Figure~\ref{fig:kl_vs_acc}. We apply the pipeline without modification: expert and shared layers are partitioned into groups exactly as dense layers are, sensitivity is estimated with the same multi-bitwidth Shapley procedure, and allocation uses the same ILP under the same fidelity constraints. No MoE-specific handling (expert-aware grouping, routing-aware calibration, or per-expert bit budgets) was introduced.

\begin{table}[H]
\centering
\caption{SLQ on Qwen3.5-35B-A3B (35B total, 3B active). Evaluated with the LM Evaluation Harness under its log-likelihood protocol, so absolute scores are not directly comparable to the generative results of Table~\ref{tab:main-accuracy}; recovery rates are. Both configurations lie within sampling variance of the BF16 baseline on every benchmark.}
\label{tab:moe-results}
\small
\begin{tabular}{@{}lccc@{}}
\toprule
Task & BF16 & SLQ-TL (3.73 bpp) & SLQ-DL (5.75 bpp) \\
\midrule
ARC-Easy       & 0.8392 & 0.8325 & 0.8405 \\
ARC-Challenge  & 0.5887 & 0.5836 & 0.5904 \\
HellaSwag      & 0.6234 & 0.6225 & 0.6241 \\
PIQA           & 0.8281 & 0.8264 & 0.8254 \\
WinoGrande     & 0.7388 & 0.7474 & 0.7451 \\
MMLU           & 0.8266 & 0.8265 & 0.8276 \\
GSM8K          & 0.8491 & 0.8901 & 0.8355 \\
\midrule
Mean recovery  & -- & \textbf{99.8\%} & \textbf{100.1\%} \\
\bottomrule
\end{tabular}
\end{table}

Table~\ref{tab:moe-results} reports per-task results over the seven benchmarks we ran for this model. SLQ-TL reaches 99.8\% mean recovery at 3.73\,bpp and SLQ-DL reaches 100.1\% at 5.75\,bpp. The per-task spread behaves as it does for the dense models: individual benchmarks fluctuate in both directions around the baseline (GSM8K is the widest, in line with its known run-to-run variance under stochastic sampling), while the aggregate stays within sampling variance. The resulting bitwidths are also consistent with the dense results of Table~\ref{tab:main-accuracy}, falling inside the 3.3--4.7\,bpp TL and 5.0--6.6\,bpp DL ranges observed there, which suggests the losslessness thresholds are governed by the fidelity target rather than by whether computation is dense or sparsely routed.

\paragraph{Actual average bits per parameter.}
\label{app:actual_bitwidths}
We report the exact average number of bits per parameter to enable precise comparisons at a given effective bitwidth. For each configuration, the reported value accounts for the true storage cost of every layer, including all quantization-related overheads. As a result, the averages reflect the realized per-parameter bitwidth rather than the nominal setting.

\begin{table}[H]
\centering
\caption{Actual bitwidth by format and nominal bitwidth}
\begin{tabular}{@{}c|cc|cc@{}}
\toprule
         & \multicolumn{2}{c|}{INT} & \multicolumn{2}{c}{NVFP/FP} \\
         & \multicolumn{2}{c|}{(groupsize 128)} & \multicolumn{2}{c}{(groupsize 16)} \\
Nominal  & Asym   & Sym   & Asym & Sym  \\
\midrule
4 & 4.156 & 4.125 & 5.0 & 4.5 \\
5 & 5.156 & 5.125 & --  & --  \\
6 & 6.156 & 6.125 & --  & --  \\
7 & 7.156 & 7.125 & --  & --  \\
8 & 8.156 & 8.125 & 8.0 & 8.0 \\
\bottomrule
\end{tabular}
\end{table}

\subsection{Token-Level Divergence Analysis}
\label{app:token-divergence}

To visualize the practical impact of distribution-lossless (DL) compression,
we perform a token-level comparison between the base Qwen3-8B (FP16)
and the SLQ-DL quantized model (5.70\,bpw, EAR~$= 99\%$).
We generate responses with the base model using greedy decoding,
then teacher-force the identical token sequence through the EAR model
and record every position where the EAR model's greedy next-token prediction differs.
Divergent tokens appear
highlighted as {\small\colorbox{tokredlight}{\color{tokred}\texttt{base}}}
with the EAR model's alternative shown as a {\color{altblue}\textsuperscript{blue}} superscript.

\paragraph{Methodology.}
The goal of this analysis is to measure how often the quantized model
would choose a different token than the base model at each generation step,
i.e.\ how frequently quantization causes the model to ``start diverging''
from the original output.
To measure this cleanly, we use \emph{teacher-forced comparison}:
we first generate a full response from the base model using greedy decoding,
then feed that exact token sequence as input to the EAR model
and record, at every position, whether the EAR model's greedy top-1 prediction
matches the base model's token.
Crucially, the EAR model is always conditioned on the base model's actual tokens,
not its own previous predictions.
This is necessary because in free (autoregressive) generation,
a single divergent token shifts the entire continuation,
making it impossible to attribute later differences to quantization error
versus the natural consequence of conditioning on a different prefix.
Teacher-forcing isolates the per-position effect of quantization:
each reported divergence reflects a position where quantization alone
caused the model to prefer a different token, independent of any
earlier divergences.

\paragraph{General statistics.}
We evaluate 15~diverse prompts spanning code generation, mathematical derivation,
essay writing, logic puzzles, and everyday questions, producing 6{,}531~tokens in total.
Of these, 6{,}407~tokens (98.1\%) are identical between both models;
only 124~tokens (1.9\%) diverge.
Per-prompt agreement ranges from 96.1\% to 99.6\%
with a median of 98.2\%.
Divergences are distributed roughly uniformly across token positions
(mean position 235 out of 512),
confirming that quantization errors do not compound over longer sequences.
All observed divergences are minor stylistic variations---synonym swaps,
formatting changes, or punctuation differences---and none alter factual content.

\paragraph{Most frequently divergent tokens.}
Since 124~divergent tokens are distributed across 6{,}531~generated tokens,
no single token dominates: the most frequent base tokens appear at most 3~times.
Table~\ref{tab:top-divergent} lists the 10~most common base tokens involved in divergences
together with their most frequent substitutions.

\begin{table}[h]
\centering
\caption{Most frequently divergent base tokens across 15~prompts.
Each row shows a base-model token that the EAR model replaces, its occurrence count, and representative substitution(s).}
\label{tab:top-divergent}
\small
\begin{tabular}{@{}lrl@{}}
\toprule
\textbf{Base token} & \textbf{Count} & \textbf{EAR substitution(s)} \\
\midrule
    \texttt{The} & 3 & \texttt{\#\#\#}, \texttt{**}, \texttt{Inside} \\
    \texttt{ (} & 3 & \texttt{ of}, \texttt{,} \\
    \texttt{ the} & 3 & \texttt{ a}, \texttt{ which}, \texttt{ it} \\
    \texttt{ and} & 2 & \texttt{ through}, \texttt{,} \\
    \texttt{ genetic} & 2 & \texttt{ instructions}, \texttt{ the} \\
    \texttt{ of} & 2 & \texttt{ on}, \texttt{:} \\
    \texttt{.\textbackslash{}n} & 2 & \texttt{ of}, \texttt{ into} \\
    \texttt{,} & 2 & \texttt{ to}, \texttt{\textbackslash{}n} \\
    \texttt{ to} & 2 & \texttt{.}, \texttt{-{}-} \\
    \texttt{ **} & 2 & \texttt{ logic}, \texttt{ the} \\
    \bottomrule
\end{tabular}
\end{table}

\paragraph{Per-prompt summary.}
Table~\ref{tab:token-divergence} lists agreement for all 15~prompts.

\begin{table}[h]
\centering
\caption{Token-level agreement between base Qwen3-8B (FP16) and the 99\% EAR model (5.70\,bpw) across 15 diverse prompts with greedy decoding.}
\label{tab:token-divergence}
\scriptsize
\begin{tabular}{@{}p{3.1cm}lrrr@{}}
\toprule
\textbf{Prompt} & \textbf{Type} & \textbf{Tok.} & \textbf{Agr.} & \textbf{Div.} \\
\midrule
    Recipe for cookies & Normal & 268 & 99.6\% & 1 \\
    Photosynthesis & Q\&A & 138 & 99.3\% & 1 \\
    Binary search tree & Code & 512 & 99.2\% & 4 \\
    Determinant intuition & Reason. & 512 & 99.0\% & 5 \\
    SQL top-5 customers & Code & 446 & 98.9\% & 5 \\
    Geometric series sum & Reason. & 512 & 98.8\% & 6 \\
    Merge intervals & Code & 448 & 98.4\% & 7 \\
    Wolf, goat, cabbage & Logic & 512 & 98.2\% & 9 \\
    FDR vs.\ Reagan & Essay & 512 & 98.0\% & 10 \\
    Python vs.\ Java & Code & 512 & 97.8\% & 11 \\
    Themes of isolation & Essay & 512 & 97.7\% & 12 \\
    Byzantine Generals & Reason. & 512 & 97.7\% & 12 \\
    Romeo and Juliet & Summ. & 259 & 96.9\% & 8 \\
    Car engine for kids & Q\&A & 364 & 96.4\% & 13 \\
    mRNA vaccines & Reason. & 512 & 96.1\% & 20 \\
    \midrule
    \textbf{Total (15 prompts)} & & \textbf{6{,}531} & \textbf{98.1\%} & \textbf{124} \\
    \bottomrule
\end{tabular}
\end{table}

\paragraph{Representative examples.}
Three examples follow, chosen to cover different domains.
All three exceed 96\% token-level agreement.

\vspace{2mm}

\noindent\textbf{Example 1} \textit{(normal text, 364 tokens, 13 divergent, 96.4\% agreement)}
\par\vspace{1mm}
\begin{promptframe}
Explain how a car engine works to a 10-year-old.
\end{promptframe}
\vspace{-2mm}
\begin{responseframe}
Sure! Let me explain how a car engine works in a way that's easy to understand, like you're\tokdiff{ }{ learning}10 years old.  Imagine your car is like a big toy that runs on gas, like the kind you put in your gas tank. Inside the car, there's something called an **engine**. The engine is like the heart of the car\tokdiff{---it}{.} helps it move.  \tokdiff{The}{Inside} engine has a bunch of small parts that move around, kind of like the wheels on a toy car. These parts are called **\tokdiff{c}{p}ylinders**. Each cylinder is like a tiny\tokdiff{ cylinder}{ chamber} (\tokdiff{that}{like}'s why they're called cylinders!)\tokdiff{ that}{ where} moves up and down.  Here's what happens:  1. **Gas is put in**: You put gas in the tank, and it goes to the engine. 2. **Spark plug lights it up**: Inside the cylinder, there's a spark plug that lights the gas on fire, like a tiny fire\tokdiff{cr}{ in}acker. 3. **Fire makes power**: When the gas burns, it makes a\tokdiff{ lot}{ big} of heat and pressure. This pressure pushes the cylinder down, which turns a rod connected to the wheels. 4. **Wheels turn**: The rod turns a big wheel called the **crankshaft**, which makes the wheels of the car turn. 5. **Repeat\tokdiff{**:}{ over} The process repeats over and over again, like a big, powerful\tokdiff{ heartbeat}{ toy}, and that's how the car moves!  So, the engine is like a big, powerful machine that uses gas to make the car go. It's kind of like a magic box that turns gas into motion!  Want to know what happens when the car stops?\tokdiff{ Then}{ The} the engine stops, and the car doesn't move anymore. But when you press the gas pedal,\tokdiff{ the}{ it} engine starts working again!  Does that make sense?
\end{responseframe}

\vspace{3mm}

\noindent\textbf{Example 2} \textit{(mathematical derivation, 320 tokens shown, 2 divergent, 98.8\% agreement)}
\par\vspace{1mm}
\begin{promptframe}
Derive the formula for the sum of an infinite geometric series. Under what conditions does the series converge?
\end{promptframe}
\vspace{-2mm}
\begin{responseframe}
To **derive the formula for the sum of an infinite geometric series**, we start with the definition of a geometric series.  ---  \#\#\# **Definition of a Geometric Series**  A **geometric series** is a series of the form:  \$\$ S = a + ar + ar\textasciicircum{}2 + ar\textasciicircum{}3 + \textbackslash{}cdots \$\$  Where: - \$ a \$ is the **first term**, - \$ r \$ is the **common ratio** between successive terms.  This series continues **indefinitely**, so it is an **infinite geometric series**.  ---  \#\#\# **Der\tokdiff{ \$}{ as} the Formula**  Lets denote the sum of the infinite geometric series as \$ S \$:  \$\$ S = a + ar + ar\textasciicircum{}2 + ar\textasciicircum{}3 + \textbackslash{}cdots \$\$  Now, multiply both sides of the equation by the common ratio \$ r \$:  \$\$ rS = ar + ar\textasciicircum{}2 + ar\textasciicircum{}3 + ar\textasciicircum{}4 + \textbackslash{}cdots \$\$  Now subtract the second equation from the first:  \$\$ S - rS = (a + ar + ar\textasciicircum{}2 + ar\textasciicircum{}3 + \textbackslash{}cdots) - (ar + ar\textasciicircum{}2 + ar\textasciicircum{}3 + \textbackslash{}cdots) \$\$  On the right-hand side, all terms cancel except the first term \$ a \$:  \$\$ S(1 - r) = a \$\$  Now solve for \$ S \$:  \$\$ S = \textbackslash{}frac\{a\}\{1 - r\} \$\$  ---  \#\#\#  **Final\tokdiff{valid}{only}:**  \$\$ \textbackslash{}boxed\{S = [\ldots]
\end{responseframe}

\vspace{3mm}

\noindent\textbf{Example 3} \textit{(distributed systems/reasoning, 280 tokens shown, 7 divergent, 97.7\% agreement)}
\par\vspace{1mm}
\begin{promptframe}
Explain the Byzantine Generals Problem in distributed computing. Why is it impossible to solve with one traitorous general?
\end{promptframe}
\vspace{-2mm}
\begin{responseframe}
The **Byzantine Generals Problem** is a classic problem in **distributed computing** that illustrates the challenges\tokdiff{ consensus}{ **} achieving **consensus** in a network of processes (or "generals") when\tokdiff{ is}{ was} of them may\tokdiff{ story}{ historical} **faulty** or **malicious**. It was introduced by **Leslie Lamport, Robert Shostak, and Marshall Pease** in a 1982 paper titled *"The Byzantine Generals Problem."*  ---  \#\#  **The Problem: The Byzantine Generals**  Imagine a group of generals, each commanding a division of the army, surrounding a city. They must decide whether to **attack** or **retreat**. However, they can only communicate with each other through **messengers**, and some of the generals may be **traitors** (i.e., they may send false messages or act against the interests of the group).  The goal is for all loyal generals\tokdiff{**}{ by} **agree on the same plan** (attack or retreat), even if some are traitors. The problem is to determine whether this is possible,\tokdiff{ communication}{ messages} under what conditions.  ---  \#\#  **Key Assumptions**  - There are **n** generals (processes). - At most **m** of them are **traitors** (faulty or malicious). -\tokdiff{ There}{ Nodes}\tokdiff{ are}{ is} **synchronous** (messages are delivered in a finite [\ldots]
\end{responseframe}

\section{Thinking Token Inflation Under Quantization}
\label{app:thinking-tokens}

We measure whether quantized models generate more thinking tokens (\texttt{<think>...</think>}) than the BF16 baseline to reach the same answer on Qwen3-8B, GSM8K (1319 problems).

\noindent\textbf{Setup.}
Inference uses vLLM (v0.17.1) with identical sampling across all configs (temperature=0.6, top\_p=0.95, top\_k=20). Each query is run 3 times (seeds 43, 44, 45) for within-query variance, yielding 3957 generations per config. Maximum generation length is 8192 tokens with context length 16384. Quantized weights are fake-quantized (dequantized) with asymmetric integer quantization, group size 128. Hardware: NVIDIA L40S (46\,GB).

\noindent\textbf{Aggregate statistics.}
Table~\ref{tab:think-aggregate} reports per-config statistics over all 3957 generations.

\begin{table}[H]
\centering
\caption{Thinking token statistics on GSM8K (Qwen3-8B, 1319 problems $\times$ 3 seeds).}
\label{tab:think-aggregate}
\small
\begin{tabular}{@{}lcccccc@{}}
\toprule
Config & Bits/param & Mean & Median & Std & Accuracy & Truncated \\
\midrule
BF16      & 16.00 & 2051 & 1548 & 1509 & 94.1\% & 40 \\
SLQ-DL    & 5.70  & 2051 & 1555 & 1518 & 93.8\% & 53 \\
SLQ-TL    & 4.72  & 2097 & 1619 & 1504 & 94.7\% & 50 \\
Uniform-4 & 4.00  & 2111 & 1627 & 1508 & 94.6\% & 54 \\
\bottomrule
\end{tabular}
\end{table}

\noindent\textbf{Paired comparison vs BF16.}
Per-query means are computed by averaging over 3 seeds, and compared to BF16 via paired Wilcoxon signed-rank test (Table~\ref{tab:think-paired}).

\begin{table}[H]
\centering
\caption{Paired thinking token comparison vs BF16 (N=1319 queries).}
\label{tab:think-paired}
\small
\begin{tabular}{@{}lcccc@{}}
\toprule
Config & Mean diff & Median ratio & Wilcoxon $p$ & Cohen's $d$ \\
\midrule
SLQ-DL    & $-1$  & 0.995 & $8.6 \times 10^{-1}$ & $-0.001$ \\
SLQ-TL    & $+45$ & 1.037 & $3.7 \times 10^{-7}$ & $+0.076$ \\
Uniform-4 & $+59$ & 1.036 & $5.2 \times 10^{-6}$ & $+0.097$ \\
\bottomrule
\end{tabular}
\end{table}

\noindent\textbf{Jointly-correct queries.}
We restrict to queries where both BF16 and the quantized model answer correctly on all 3 runs, testing whether inflation persists even when both models reach the right answer (Table~\ref{tab:think-joint}).

\begin{table}[H]
\centering
\caption{Thinking tokens on jointly-correct queries.}
\label{tab:think-joint}
\small
\begin{tabular}{@{}lccccc@{}}
\toprule
Config & N queries & BF16 mean & Quant mean & Diff & Wilcoxon $p$ \\
\midrule
SLQ-DL    & 1195 & 1800 & 1798 & $-2$  & $6.3 \times 10^{-1}$ \\
SLQ-TL    & 1195 & 1789 & 1839 & $+50$ & $2.4 \times 10^{-7}$ \\
Uniform-4 & 1192 & 1799 & 1861 & $+62$ & $2.0 \times 10^{-6}$ \\
\bottomrule
\end{tabular}
\end{table}

\noindent\textbf{Thinking tokens by correctness.}
Wrong answers consume 2.4--2.6$\times$ more thinking tokens across all configs (Table~\ref{tab:think-correctness}).

\begin{table}[H]
\centering
\caption{Mean thinking tokens by answer correctness.}
\label{tab:think-correctness}
\small
\begin{tabular}{@{}lccc@{}}
\toprule
Config & Correct mean & Wrong mean & Ratio (wrong/correct) \\
\midrule
BF16      & 1893 & 4564 & 2.41 \\
SLQ-DL    & 1882 & 4605 & 2.45 \\
SLQ-TL    & 1934 & 5003 & 2.59 \\
Uniform-4 & 1945 & 4987 & 2.56 \\
\bottomrule
\end{tabular}
\end{table}

\noindent\textbf{Within-query variance.}
Coefficient of variation (CV = std\,/\,mean) across 3 runs of the same query (Table~\ref{tab:think-cv}). The 2--3\% inflation from lossy quantization is small relative to natural run-to-run variance (${\sim}19\%$), but statistically significant with $N{=}1319$ paired comparisons.

\begin{table}[H]
\centering
\caption{Within-query coefficient of variation across 3 seeds.}
\label{tab:think-cv}
\small
\begin{tabular}{@{}lcc@{}}
\toprule
Config & Mean CV & Median CV \\
\midrule
BF16      & 18.8\% & 16.9\% \\
SLQ-DL    & 18.8\% & 16.9\% \\
SLQ-TL    & 19.3\% & 17.3\% \\
Uniform-4 & 19.5\% & 17.3\% \\
\bottomrule
\end{tabular}
\end{table}

\noindent\textbf{Token budget breakdown.}
The inflation is entirely in the thinking portion; non-thinking output length is stable across configs (Table~\ref{tab:think-budget}).

\begin{table}[H]
\centering
\caption{Mean token budget breakdown per generation.}
\label{tab:think-budget}
\small
\begin{tabular}{@{}lcccc@{}}
\toprule
Config & Think & Non-think & Total & Think \% \\
\midrule
BF16      & 2051 & 305 & 2356 & 87.1\% \\
SLQ-DL    & 2051 & 306 & 2356 & 87.0\% \\
SLQ-TL    & 2097 & 302 & 2399 & 87.4\% \\
Uniform-4 & 2111 & 286 & 2397 & 88.1\% \\
\bottomrule
\end{tabular}
\end{table}

\noindent\textbf{Thinking token percentiles.}
Table~\ref{tab:think-percentiles} shows the distribution across generations.

\begin{table}[H]
\centering
\caption{Thinking token percentiles across all generations.}
\label{tab:think-percentiles}
\small
\begin{tabular}{@{}lccccccc@{}}
\toprule
Config & P10 & P25 & P50 & P75 & P90 & P95 & P99 \\
\midrule
BF16      & 803 & 1085 & 1548 & 2462 & 3993 & 5362 & 8179 \\
SLQ-DL    & 802 & 1073 & 1555 & 2466 & 3944 & 5371 & 8192 \\
SLQ-TL    & 841 & 1101 & 1619 & 2490 & 4061 & 5337 & 8192 \\
Uniform-4 & 831 & 1117 & 1627 & 2554 & 4062 & 5408 & 8192 \\
\bottomrule
\end{tabular}
\end{table}

\subsection{Detailed Quantization Results}
\label{app:detailed-results}

This section presents comprehensive quantization results across bitwidths, number formats, and quantization schemes. All results are evaluated on Llama-3.1-8B-Instruct using 512 calibration samples. The metrics reported here (EAR, KL, Margin, PPL) are distributional fidelity measures used for distribution-lossless (DL) compression. Task-lossless (TL) compression targets benchmark recovery instead and can be achieved at lower bitwidths; see the main results in Section~\ref{sec:accuracy}.

\paragraph{Metrics.} We report four key metrics for quantization quality:
\begin{itemize}
\item \textbf{EAR}: The expected acceptance rate, measuring the overlap in output probability mass between BF16 and quantized models. This equals $\sum_k \min(p_k, q_k)$ where $p$ and $q$ are the top-10 token distributions. Higher EAR indicates closer alignment to the original output distribution.
\item \textbf{Margin at Disagreement}: The average probability gap on tokens where the quantized and BF16 models disagree on their top prediction. Lower margin means flips occur only on close calls, indicating smaller quantization error.
\item \textbf{KL Divergence}: The KL divergence between BF16 and quantized output distributions on the calibration set.
\item \textbf{Perplexity (PPL)}: The perplexity of the quantized model relative to BF16, computed from the calibration set.
\end{itemize}

\paragraph{Format and Bitwidth Notations.} All tables use asymmetric integer quantization (Asym) or symmetric integer quantization (Sym) with group size 128. Bitwidth values (e.g., ``4.50'') include the overhead from scales and zero-points. The format notation ``4-8'' means only bitwidths 4 and 8 are available for allocation, while ``4-5-6-7-8'' means all intermediate bitwidths are available.

\paragraph{Uniform INT4 baselines.} Table~\ref{tab:int4-ref} shows uniform 4-bit integer quantization results for asymmetric and symmetric schemes. These serve as single-bitwidth reference points.

\paragraph{Weights+Activations (W+A) Quantization.} Table~\ref{tab:wa-results} compares weight-only (W) quantization with combined weights+activations (W+A) quantization, both obtained via evolutionary search with adaptive curation (Appendix~\ref{app:evo}) for Llama 3.1 8B Instruct. Weight-only quantization at 5.42 bits achieves excellent recovery across all tasks (97.6\%--100.5\%), with GSM8K, PIQA, and Winogrande exceeding baseline accuracy. Quantizing both weights and activations at a comparable constraint requires a higher bitwidth of 6.50--6.97 bits. At 6.97 bits (Max KL 0.01), the W+A configuration recovers 99.4\%--100.1\% on all tasks. This shows that activations are more sensitive to quantization than weights, requiring additional precision to maintain the same quality level, but near-lossless W+A quantization remains feasible.

\begin{table*}[!ht]
\centering
\caption{INT4 Reference Quantization Results}
\label{tab:int4-ref}
\small
\begin{tabular}{@{}lccccc@{}}
\toprule
Method & BW & EAR $\uparrow$ & Margin $\downarrow$ & KL $\downarrow$ & PPL $\uparrow$ \\
\midrule
Uniform (Asym) & 4.156 & 0.958 & 0.113 & 0.0369 & 0.956 \\
Uniform (Sym) & 4.156 & 0.946 & 0.148 & 0.0576 & 0.926 \\
\bottomrule
\end{tabular}
\end{table*}

\begin{table*}[!ht]
\centering
\caption{Weights+Activations Quantization Results (Evolutionary Search with Adaptive Curation) for Llama 3.1 8B Instruct}
\label{tab:wa-results}
\resizebox{\textwidth}{!}{
\begin{tabular}{@{}lccccccccc@{}}
\toprule
Method & Type & Max KL & BW & PPL & HellaSwag $\uparrow$ & Winogrande $\uparrow$ & PIQA $\uparrow$ & GSM8K $\uparrow$ & ARC-C $\uparrow$ \\
\midrule
FP16 & -- & -- & 16 & 9.10 & 57.41 & 67.48 & 79.76 & 83.17 & 52.56 \\
\midrule
Evo Search & W & 0.0125 & 5.42 & 9.35 & 57.37 (99.9\%) & 67.80 (100.5\%) & 80.20 (100.5\%) & 83.55 (100.5\%) & 51.28 (97.6\%) \\
\midrule
Evo Search & W+A & 0.015 & 6.50 & 9.37 & 57.02 (99.3\%) & 67.25 (99.6\%) & 79.38 (99.5\%) & 82.03 (98.6\%) & 51.02 (97.1\%) \\
Evo Search & W+A & 0.010 & 6.97 & 9.33 & 57.16 (99.6\%) & 67.48 (100.0\%) & 79.81 (100.1\%) & 83.02 (99.8\%) & 52.22 (99.4\%) \\
\bottomrule
\end{tabular}
}
\end{table*}

\paragraph{Integer quantization with binary bitwidth range.} Tables~\ref{tab:int-asym-4-8} and~\ref{tab:int-sym-4-8} show integer asymmetric and symmetric quantization when only bitwidths 4 and 8 are available for allocation. Uniform baselines at each single bitwidth are included for reference. Linear and Shapley rows show non-uniform allocations at matched average bitwidths.

\begin{table*}[p]
\centering
\caption{INT Asymmetric Quantization Results (4-8 Config)}
\label{tab:int-asym-4-8}
\small
\begin{tabular}{@{}lccccc@{}}
\toprule
Method & BW & EAR $\uparrow$ & Margin $\downarrow$ & KL $\downarrow$ & PPL $\uparrow$ \\
\midrule
\multicolumn{6}{l}{\textit{Uniform Baselines}} \\
\midrule
Uniform & 4.16 & 0.938 & 0.148 & 0.0277 & 0.970 \\
Uniform & 5.16 & 0.964 & 0.083 & 0.0068 & 0.992 \\
Uniform & 6.16 & 0.988 & 0.029 & 0.0017 & 0.998 \\
Uniform & 7.16 & 0.989 & 0.023 & 0.0004 & 1.000 \\
Uniform & 8.16 & 0.992 & 0.021 & 0.0001 & 1.000 \\
\midrule
\multicolumn{6}{l}{\textit{4.5-bit}} \\
\midrule
Linear & 4.50 & \textbf{0.967} & \textbf{0.051} & \textbf{0.0196} & 0.978 \\
Shapley & 4.50 & 0.952 & 0.130 & 0.0199 & \textbf{0.979} \\
\midrule
\multicolumn{6}{l}{\textit{5.5-bit}} \\
\midrule
Linear & 5.50 & \textbf{0.975} & \textbf{0.047} & 0.0145 & \textbf{0.984} \\
Shapley & 5.50 & 0.975 & \textbf{0.047} & \textbf{0.0143} & 0.984 \\
\midrule
\multicolumn{6}{l}{\textit{6.0-bit}} \\
\midrule
Linear & 5.99 & 0.977 & 0.051 & 0.0126 & 0.986 \\
Shapley & 5.99 & \textbf{0.981} & \textbf{0.047} & \textbf{0.0122} & \textbf{0.986} \\
\midrule
\multicolumn{6}{l}{\textit{6.5-bit}} \\
\midrule
Linear & 6.50 & 0.983 & 0.043 & 0.0098 & 0.989 \\
Shapley & 6.50 & \textbf{0.984} & \textbf{0.029} & \textbf{0.0091} & \textbf{0.990} \\
\midrule
\multicolumn{6}{l}{\textit{7.0-bit}} \\
\midrule
Linear & 6.98 & \textbf{0.986} & 0.043 & 0.0073 & \textbf{0.992} \\
Shapley & 6.98 & 0.984 & \textbf{0.029} & \textbf{0.0069} & 0.991 \\
\bottomrule
\end{tabular}
\end{table*}

\begin{table*}[p]
\centering
\caption{INT Symmetric Quantization Results (4-8 Config)}
\label{tab:int-sym-4-8}
\small
\begin{tabular}{@{}lccccc@{}}
\toprule
Method & BW & EAR $\uparrow$ & Margin $\downarrow$ & KL $\downarrow$ & PPL $\uparrow$ \\
\midrule
\multicolumn{6}{l}{\textit{Uniform Baselines}} \\
\midrule
Uniform & 4.12 & 0.938 & 0.145 & 0.0344 & 0.957 \\
Uniform & 5.12 & 0.964 & 0.083 & 0.0087 & 0.991 \\
Uniform & 6.12 & 0.985 & 0.051 & 0.0022 & 0.997 \\
Uniform & 7.12 & 0.989 & 0.023 & 0.0006 & 0.999 \\
Uniform & 8.12 & 0.997 & 0.005 & 0.0001 & 1.000 \\
\midrule
\multicolumn{6}{l}{\textit{4.5-bit}} \\
\midrule
Linear & 4.50 & \textbf{0.955} & \textbf{0.124} & \textbf{0.0236} & \textbf{0.971} \\
Shapley & 4.50 & 0.953 & 0.132 & 0.0250 & 0.970 \\
\midrule
\multicolumn{6}{l}{\textit{5.5-bit}} \\
\midrule
Linear & 5.50 & \textbf{0.967} & \textbf{0.091} & 0.0175 & 0.979 \\
Shapley & 5.50 & 0.966 & 0.117 & \textbf{0.0172} & \textbf{0.979} \\
\midrule
\multicolumn{6}{l}{\textit{6.0-bit}} \\
\midrule
Linear & 6.00 & 0.969 & 0.110 & 0.0151 & 0.982 \\
Shapley & 6.00 & \textbf{0.970} & \textbf{0.091} & \textbf{0.0145} & \textbf{0.983} \\
\midrule
\multicolumn{6}{l}{\textit{6.5-bit}} \\
\midrule
Linear & 6.50 & 0.974 & 0.124 & 0.0116 & 0.986 \\
Shapley & 6.50 & \textbf{0.983} & \textbf{0.059} & \textbf{0.0107} & \textbf{0.987} \\
\midrule
\multicolumn{6}{l}{\textit{7.0-bit}} \\
\midrule
Linear & 6.98 & \textbf{0.986} & \textbf{0.043} & 0.0084 & 0.990 \\
Shapley & 6.98 & 0.980 & 0.059 & \textbf{0.0079} & \textbf{0.990} \\
\bottomrule
\end{tabular}
\end{table*}

\paragraph{Integer quantization with full bitwidth range.} Tables~\ref{tab:int-asym-4-5-6-7-8} and~\ref{tab:int-sym-4-5-6-7-8} show integer asymmetric and symmetric quantization with the full bitwidth set $\{4,5,6,7,8\}$. The additional intermediate bitwidths give the solver finer-grained control over per-group allocation.

\begin{table*}[p]
\centering
\caption{INT Asymmetric Quantization Results (4-5-6-7-8 Config)}
\label{tab:int-asym-4-5-6-7-8}
\small
\begin{tabular}{@{}lccccc@{}}
\toprule
Method & BW & EAR $\uparrow$ & Margin $\downarrow$ & KL $\downarrow$ & PPL $\uparrow$ \\
\midrule
\multicolumn{6}{l}{\textit{Uniform Baselines}} \\
\midrule
Uniform & 4.16 & 0.938 & 0.148 & 0.0277 & 0.970 \\
Uniform & 5.16 & 0.964 & 0.083 & 0.0068 & 0.992 \\
Uniform & 6.16 & 0.988 & 0.029 & 0.0017 & 0.998 \\
Uniform & 7.16 & 0.989 & 0.023 & 0.0004 & 1.000 \\
Uniform & 8.16 & 0.992 & 0.021 & 0.0001 & 1.000 \\
\midrule
\multicolumn{6}{l}{\textit{4.5-bit}} \\
\midrule
Linear & 4.50 & \textbf{0.964} & \textbf{0.074} & \textbf{0.0171} & 0.980 \\
Shapley & 4.50 & 0.960 & 0.115 & 0.0177 & \textbf{0.982} \\
\midrule
\multicolumn{6}{l}{\textit{5.5-bit}} \\
\midrule
Linear & 5.50 & 0.988 & 0.045 & \textbf{0.0041} & \textbf{0.996} \\
Shapley & 5.50 & \textbf{0.991} & \textbf{0.027} & 0.0047 & 0.995 \\
\midrule
\multicolumn{6}{l}{\textit{6.0-bit}} \\
\midrule
Linear & 6.00 & 0.989 & \textbf{0.027} & \textbf{0.0025} & 0.997 \\
Shapley & 6.00 & \textbf{0.991} & 0.029 & 0.0026 & \textbf{0.997} \\
\midrule
\multicolumn{6}{l}{\textit{6.5-bit}} \\
\midrule
Linear & 6.50 & \textbf{0.995} & \textbf{0.023} & \textbf{0.0010} & \textbf{0.999} \\
Shapley & 6.50 & 0.994 & 0.026 & 0.0011 & 0.999 \\
\midrule
\multicolumn{6}{l}{\textit{7.0-bit}} \\
\midrule
Linear & 7.00 & \textbf{0.994} & \textbf{0.021} & \textbf{0.0006} & \textbf{1.000} \\
Shapley & 7.00 & \textbf{0.994} & \textbf{0.021} & 0.0006 & \textbf{1.000} \\
\bottomrule
\end{tabular}
\end{table*}

\begin{table*}[p]
\centering
\caption{INT Symmetric Quantization Results (4-5-6-7-8 Config)}
\label{tab:int-sym-4-5-6-7-8}
\small
\begin{tabular}{@{}lccccc@{}}
\toprule
Method & BW & EAR $\uparrow$ & Margin $\downarrow$ & KL $\downarrow$ & PPL $\uparrow$ \\
\midrule
\multicolumn{6}{l}{\textit{Uniform Baselines}} \\
\midrule
Uniform & 4.12 & 0.938 & 0.145 & 0.0344 & 0.957 \\
Uniform & 5.12 & 0.964 & 0.083 & 0.0087 & 0.991 \\
Uniform & 6.12 & 0.985 & 0.051 & 0.0022 & 0.997 \\
Uniform & 7.12 & 0.989 & 0.023 & 0.0006 & 0.999 \\
Uniform & 8.12 & 0.997 & 0.005 & 0.0001 & 1.000 \\
\midrule
\multicolumn{6}{l}{\textit{4.5-bit}} \\
\midrule
Linear & 4.50 & \textbf{0.963} & \textbf{0.112} & \textbf{0.0205} & 0.974 \\
Shapley & 4.50 & 0.952 & 0.119 & 0.0212 & \textbf{0.975} \\
\midrule
\multicolumn{6}{l}{\textit{5.5-bit}} \\
\midrule
Linear & 5.50 & 0.983 & \textbf{0.059} & \textbf{0.0051} & \textbf{0.994} \\
Shapley & 5.50 & \textbf{0.986} & 0.091 & 0.0054 & 0.993 \\
\midrule
\multicolumn{6}{l}{\textit{6.0-bit}} \\
\midrule
Linear & 6.00 & 0.983 & 0.043 & \textbf{0.0030} & \textbf{0.997} \\
Shapley & 6.00 & \textbf{0.986} & \textbf{0.025} & 0.0031 & 0.996 \\
\midrule
\multicolumn{6}{l}{\textit{6.5-bit}} \\
\midrule
Linear & 6.50 & \textbf{0.992} & \textbf{0.023} & \textbf{0.0013} & 0.999 \\
Shapley & 6.50 & 0.989 & 0.025 & 0.0013 & \textbf{0.999} \\
\midrule
\multicolumn{6}{l}{\textit{7.0-bit}} \\
\midrule
Linear & 7.00 & \textbf{0.992} & 0.015 & \textbf{0.0007} & 0.999 \\
Shapley & 7.00 & \textbf{0.992} & \textbf{0.005} & 0.0007 & \textbf{1.000} \\
\bottomrule
\end{tabular}
\end{table*}

\paragraph{Distribution-lossless (DL) constraint-based allocation.} Table~\ref{tab:int-asym-lossless} shows integer asymmetric results from the binary-search solver targeting a DL quality constraint (EAR threshold) rather than a fixed average bitwidth. Results are shown for both the binary $\{4,8\}$ and full $\{4,5,6,7,8\}$ bitwidth ranges.

\begin{table*}[p]
\centering
\caption{INT Asymmetric Distribution-Lossless (DL) Quantization Results}
\label{tab:int-asym-lossless}
\small
\begin{tabular}{@{}llccccc@{}}
\toprule
Config & Method & BW & EAR $\uparrow$ & Margin $\downarrow$ & KL $\downarrow$ & PPL $\uparrow$ \\
\midrule
\multicolumn{7}{l}{\textit{Uniform Baseline}} \\
\midrule
-- & Uniform & 4.16 & 0.938 & 0.148 & 0.0277 & 0.970 \\
\midrule
\multicolumn{7}{l}{\textit{Lossless Constraint-Based}} \\
\midrule
4-8 & Shapley & 7.74 & 0.981 & 0.027 & 0.0020 & 0.997 \\
4-5-6-7-8 & Shapley & 5.59 & \textbf{0.994} & \textbf{0.020} & 0.0039 & 0.996 \\
4-5-6-7-8 & Shapley & 6.19 & 0.989 & 0.024 & \textbf{0.0018} & \textbf{0.998} \\
\bottomrule
\end{tabular}
\end{table*}

\paragraph{Floating-point quantization.} Tables~\ref{tab:fp-asym-4-8} and~\ref{tab:fp-sym-4-8} use a custom floating-point format based on NVFP4 with an added zeropoint for the asymmetric variant. Only bitwidths 4 and 8 are available. These tables complement the integer results above by showing the same asymmetric vs.\ symmetric comparison in the FP domain.

\begin{table*}[p]
\centering
\caption{FP Asymmetric Quantization Results (4-8 Config)}
\label{tab:fp-asym-4-8}
\small
\begin{tabular}{@{}lccccc@{}}
\toprule
Method & BW & EAR $\uparrow$ & Margin $\downarrow$ & KL $\downarrow$ & PPL $\uparrow$ \\
\midrule
\multicolumn{6}{l}{\textit{Uniform Baselines}} \\
\midrule
Uniform & 5.00 & 0.955 & 0.130 & 0.0264 & 0.970 \\
Uniform & 8.00 & 0.989 & 0.024 & 0.0016 & 0.997 \\
\midrule
\multicolumn{6}{l}{\textit{5.5-bit}} \\
\midrule
Linear & 5.50 & \textbf{0.972} & \textbf{0.066} & 0.0168 & 0.980 \\
Shapley & 5.50 & 0.969 & 0.088 & \textbf{0.0168} & \textbf{0.981} \\
\midrule
\multicolumn{6}{l}{\textit{6.0-bit}} \\
\midrule
Linear & 6.00 & \textbf{0.977} & \textbf{0.051} & \textbf{0.0139} & \textbf{0.984} \\
Shapley & 6.00 & 0.975 & 0.088 & 0.0140 & 0.983 \\
\midrule
\multicolumn{6}{l}{\textit{6.5-bit}} \\
\midrule
Linear & 6.49 & \textbf{0.980} & \textbf{0.042} & \textbf{0.0115} & 0.987 \\
Shapley & 6.50 & 0.977 & 0.051 & 0.0116 & \textbf{0.987} \\
\midrule
\multicolumn{6}{l}{\textit{7.0-bit}} \\
\midrule
Linear & 6.99 & 0.981 & 0.042 & \textbf{0.0083} & 0.990 \\
Shapley & 6.99 & \textbf{0.985} & \textbf{0.029} & 0.0084 & \textbf{0.990} \\
\bottomrule
\end{tabular}
\end{table*}

\begin{table*}[p]
\centering
\caption{FP Symmetric Quantization Results (4-8 Config)}
\label{tab:fp-sym-4-8}
\small
\begin{tabular}{@{}lccccc@{}}
\toprule
Method & BW & EAR $\uparrow$ & Margin $\downarrow$ & KL $\downarrow$ & PPL $\uparrow$ \\
\midrule
\multicolumn{6}{l}{\textit{Uniform Baselines}} \\
\midrule
Uniform & 8.00 & 0.989 & 0.024 & 0.0016 & 0.997 \\
\midrule
\multicolumn{6}{l}{\textit{5.5-bit}} \\
\midrule
Linear & 5.50 & 0.949 & 0.083 & 0.0163 & 0.983 \\
Shapley & 5.50 & \textbf{0.957} & \textbf{0.076} & \textbf{0.0161} & \textbf{0.983} \\
\midrule
\multicolumn{6}{l}{\textit{6.0-bit}} \\
\midrule
Linear & 6.00 & 0.960 & \textbf{0.076} & 0.0136 & 0.984 \\
Shapley & 6.00 & \textbf{0.971} & 0.083 & \textbf{0.0134} & \textbf{0.985} \\
\midrule
\multicolumn{6}{l}{\textit{6.5-bit}} \\
\midrule
Linear & 6.50 & 0.966 & \textbf{0.074} & \textbf{0.0107} & 0.987 \\
Shapley & 6.50 & \textbf{0.969} & 0.080 & 0.0109 & \textbf{0.987} \\
\midrule
\multicolumn{6}{l}{\textit{7.0-bit}} \\
\midrule
Linear & 7.00 & \textbf{0.974} & 0.076 & 0.0077 & 0.990 \\
Shapley & 7.00 & 0.969 & \textbf{0.074} & \textbf{0.0076} & \textbf{0.991} \\
\bottomrule
\end{tabular}
\end{table*}

\end{document}